\documentclass{article}

\PassOptionsToPackage{numbers,sort&compress}{natbib}
\usepackage[preprint]{neurips_2026}

\usepackage[utf8]{inputenc} % allow utf-8 input
\usepackage[T1]{fontenc}    % use 8-bit T1 fonts
\usepackage{hyperref}       % hyperlinks
\usepackage{url}            % simple URL typesetting
\usepackage{booktabs}       % professional-quality tables
\usepackage{amsfonts}       % blackboard math symbols
\usepackage{nicefrac}       % compact symbols for 1/2, etc.
\usepackage{microtype}      % microtypography
\usepackage{xcolor}         % colors

\title{\ourmethod: Ownership Verification for Graph Neural Networks}

\author{%
  Rahul Nandakumar \\
  McCombs School of Business\\
  University of Texas at Austin\\
  Austin, TX \\
  \texttt{rahul.nandakumar@utexas.edu} \\
  \And
  Deepayan Chakrabarti \\
  McCombs School of Business\\
  University of Texas at Austin\\
  Austin, TX \\
  \texttt{deepay@utexas.edu} \\
}

\usepackage{bm, xspace, dsfont}
\usepackage{algorithm, algpseudocode}
\usepackage{xparse}
\usepackage{nicematrix}
\usepackage{amsmath, amsthm}
\usepackage{subcaption}
\usepackage{caption}
\usepackage{xcolor}

\usepackage{multirow}
\usepackage{booktabs}
\usepackage{float}
\usepackage{array}
  
\usepackage{enumitem}
\usepackage{soul}

\usepackage{stfloats}

\sethlcolor{red!30}

\makeatletter
\NewDocumentCommand{\LeftComment}{s m}{%
  \Statex \IfBooleanF{#1}{\hspace*{\ALG@thistlm}}\(\triangleright\) #2}
\makeatother

\DeclareMathOperator*{\argmin}{arg\,min}

\newcommand{\ourmethod}{\textsc{CopyCop}\xspace}

\newcommand{\ualpha}{\ensuremath\underline{\alpha}\xspace}
\newcommand{\oalpha}{\ensuremath\overline{\alpha}\xspace}

\newcommand{\bh}{{\ensuremath{\bm h}}\xspace}
\newcommand{\hbh}{{\ensuremath{\hat{\bm h}}}\xspace}

\newcommand{\bx}{{\ensuremath{\bm x}}\xspace}
\newcommand{\bv}{{\ensuremath{\bm v}}\xspace}
\newcommand{\bw}{{\ensuremath{\bm w}}\xspace}

\theoremstyle{plain}
\newtheorem{theorem}{Theorem}[section]

\newtheorem{lemma}[theorem]{Lemma}

\theoremstyle{definition}
\newtheorem{definition}[theorem]{Definition}
\newtheorem{assumption}[theorem]{Assumption}
\theoremstyle{remark}
\newtheorem{remark}[theorem]{Remark}

\begin{document}

\maketitle

\begin{abstract}
Given two GNNs that output node embeddings, how can we determine if they were trained independently?
An adversary could have trained one GNN specifically to mimic the other GNN's embeddings.
To obscure this relationship between the GNNs, the adversarial GNN might then transform its output embeddings.
The two GNNs could have different architectures, weights, and embedding dimensions, and the adversary can transform the embeddings.
Despite these stringent conditions, our algorithm (named \ourmethod) can identify such copycat GNNs, unlike existing watermarking and fingerprinting methods. 
We also provide theoretical guarantees for \ourmethod.
Finally, experiments on 14 datasets and 5 GNN architectures demonstrate that \ourmethod is accurate and robust against a broad class of adversarial attacks and transformations. Code is available at: \url{https://anonymous.4open.science/r/CopyCop-Graph-Ownership-Verification-8143/README.md}
\end{abstract}

\section{Introduction}
\label{sec:intro}

Graph Neural Networks (GNNs) are increasingly deployed as reusable embedding models, often exposed through Embeddings-as-a-Service~\citep{peng_are_2023}. In this setting, users query a model with graphs and obtain node embeddings that can be reused for downstream tasks. While this enables flexible reuse, it also raises a fundamental question: \emph{can we determine whether another model was independently trained, or derived from a deployed GNN?}

An adversary can train a surrogate GNN to mimic a victim model's embeddings by querying the victim on many graphs and learning from the resulting input-output pairs.
The adversary can then apply transformations such as rotation, scaling, permutation, or changes in embedding dimension. These transformations may substantially change the representation while preserving downstream performance, making the surrogate appear unrelated to the victim. 
Furthermore, the surrogate GNN's architecture can be different from the victim GNN.
Thus, given a victim GNN that outputs embeddings, our goal is to identify surrogate GNNs despite changes in architecture, parameters, embedding dimension, or output transformation.

Existing surrogate detection methods fall into two categories: watermarks and fingerprints. Watermarking methods modify training so that the model produces distinctive outputs on special graphs~\citep{zhao_watermarking_2021, xu_watermarking_2023, pregip}. However, an adversary can bypass such signals through model extraction: by querying the victim on ordinary graphs and training only on the resulting input-output pairs, the surrogate need not reproduce the special watermarked behavior~\citep{lukas_sok_2021}.

Fingerprinting methods avoid modifying the victim model and instead rely on intrinsic properties of its outputs. For GNNs that output embeddings, existing methods typically compare the victim and candidate embeddings directly~\citep{grove}. However, simple transformations such as rotation or changes in embedding dimension can make embeddings appear different to a classifier while preserving downstream accuracy.
As a result, existing embedding-based fingerprints are vulnerable to cosmetic changes in the surrogate representation.

\begin{figure*}[t]
    \includegraphics[width=\textwidth]{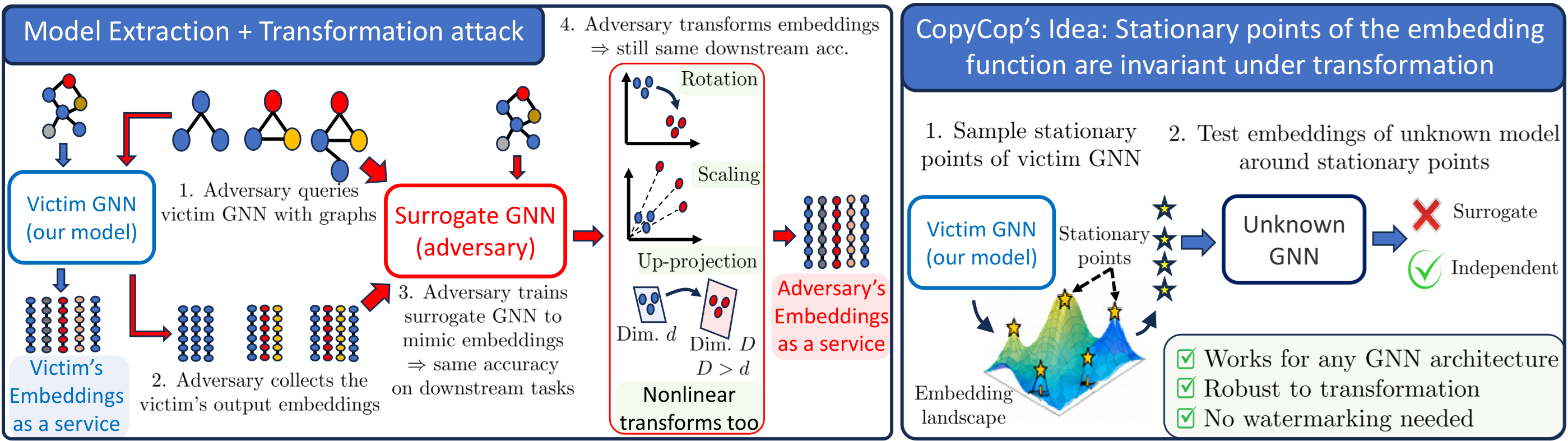}
    \caption{{\em Overview of \ourmethod.:}
    Under the embeddings-as-a-service model, the victim GNN provides embeddings for input graphs, which can then be used for downstream tasks.
    An adversary can train a surrogate GNN to mimic these embeddings and then transform them, achieving similar accuracy while obscuring the surrogate relationship.
    \ourmethod detects surrogate models under such transformations without watermarking the victim GNN.}
    \label{fig:overview}
\end{figure*}

\subsection{Our Contributions}

We propose \ourmethod, a fingerprinting method for GNNs that is robust to a wide range of adversarial transformations, including rotations, scaling, and changes in embedding dimension. Our method is architecture-agnostic: the victim and surrogate GNNs may use different architectures, parameters, and embedding dimensions. Moreover, our fingerprints are randomized and can be regenerated at any time, so leaked fingerprints do not permanently compromise the method. To our knowledge, \emph{\ourmethod is the first fingerprinting method for GNN embeddings that works under a broad class of transformations, across GNN architectures, and at any time.}

The key idea is to use stationary points of the embedding function as fingerprints. A node's embedding is a function of the graph structure and node features. Since an adversary may transform the embedding space, the fingerprint must be invariant to such transformations. We prove that stationary points of the embedding function have this invariance property. We show how to sample stationary points from the victim GNN~$M$ and test whether these points are also stationary for a candidate GNN~$Z$. This test lets us determine, with high confidence, whether $Z$ is independently trained or a surrogate of $M$.

We prove that \ourmethod detects surrogate models under reasonable conditions, irrespective of GNN architecture or embedding dimension. Empirically, we evaluate \ourmethod on 14 datasets and 5 popular GNN architectures. The results show that \ourmethod is robust to model extraction, pruning, fine-tuning, and a broad range of embedding transformations.

The rest of the paper is organized as follows.
We discuss related work in Section~\ref{sec:related}.
In Section~\ref{sec:prop}, we present our proposed method and prove its robustness against embedding transformations.
We empirically validate our approach in Section~\ref{sec:exp}, and conclude in Section~\ref{sec:conc}.
All proofs and extra experiments are deferred to the Appendix.

\section{Related Work}
\label{sec:related}
Our work on ownership verification is related to the broader area of adversarial attacks to infer the graph's structure and properties~\citep{zhang_gnnguard_2020, deng2022garnet, wang_group_2022, zhang_inference_2022}.
Robustness against backdoor attacks is also related to trustworthy AI~\citep{dai_comprehensive_2024} and integrity verification~\citep{8953972, kuttichira2022verification, wang2023publiccheck}, and query-based schemes for GNNs in MLaaS~\citep{wu2024securing}.
Here, we present an overview of the work that is most closely related to us.
%, namely, model-stealing attacks, watermarks, and fingerprints.

%\vspace{-1em}
\textbf{Model stealing attacks and defenses:}
Model stealing attacks and defenses have been studied for images~\citep{papernot_practical_2017} and text~\citep{krishna_thieves_2020}, among others.
One defense is to perturb the model's layers~\citep{lee_defending_2019}, but this fails to counter model extraction attacks~\citep{juuti_prada_2019}.
Other defenses detect model extraction attacks based on their query patterns~\citep{juuti_prada_2019}.
However, we assume that the victim GNN is public, or the adversarial queries can be successfully hidden.
Similar attacks have been proposed for GNNs~\citep{defazio_adversarial_2019, shen_model_2022}, but \citep{lukas_sok_2021} show that current defenses are insufficient.
A related but distinct problem is \emph{integrity verification}: detecting whether a deployed model has been tampered with after deployment~\citep{8953972, kuttichira2022verification, wang2023publiccheck}.
However, we focus on \emph{ownership} verification of a separately trained surrogate that may differ entirely in architecture and weights.

% Model stealing and defenses for DNN:
% For images (Papernot/2017), Text (Krishan/2020)
% Seeking patterns in adversarial queries used for model extraction (PRADA)
% Perturb the model's final activation layer, slightly altering the output probabilities. (Lee/2019) but model extraction still works (PRADA)
% Model stealing attacks on GNNs: 
% \citep{defazio_adversarial_2019, shen_model_2022}
% Survey of many attacks and defenses against them (SOK Lukas/2021)

%\vspace{-1em}
\textbf{Watermarks:}
Watermarks can be embedded in the victim model's weights by regularizing the model parameters during training~\citep{uchida_embedding_2017}.
Other approaches embed backdoors, training the model to return improbable outputs for special inputs~\citep{adi_turning_2018, darvish_rouhani_deepsigns_2019, kwon_blindnet_2022, zhao_recipe_2023}.
Backdoor-based watermarks have also been proposed for GNNs~\citep{zhao_watermarking_2021, xu_watermarking_2023, pregip}.
However, watermarks can degrade performance~\citep{grove}, and are susceptible to model extraction attacks~\citep{lukas_sok_2021}.
% \txtred{\citep{chen2021you} propose lottery-ticket-based ownership verification for DNNs.
% Their white-box schemes (V1/V2) require the surrogate to share the victim's weight mask, which does not hold when the surrogate is a separately trained GNN with a different architecture.
% Their black-box scheme (V3) is based on watermarks, which are again susceptible to model extraction attacks.}

%Lottery-ticket-based ownership verification~\citep{chen2021you} requires the surrogate to share the victim's weight mask (white-box) or relies on watermarks (black-box), both of which are ineffective against model extraction with a different architecture.
%We empirically show that \ourmethod is more robust to model extraction and embedding transformations.

% Embed in DNN by regularizing model params (Uchida)
% Backdoor samples (Kwon/Kim/2022, DeepSigns/Rouhani/2019, Adi/2018)
% Watermarks for Diffusion models (Zhao/2023)
% Embed Random graphs for GNNs targeting node classification (Zhao/2021), extend to graph classificatio (Xu/2023), then for any task (PreGIP)

% Watermarking has been shown to degrade model accuracy\citep{grove}
% model extraction attacks are the most effective removal attacks against a majority of watermarks (SOK Lukas/2021)

%\vspace{-1em}
\textbf{Fingerprints:}
Unlike watermarks, fingerprints do not modify the victim model.
Fingerprints can be based on model weights~\citep{chen_deepmarks_2019}, special inputs for which some surrogate and independent models give very different outputs~\citep{lukas_deep_2021}, pairs of samples along special vectors that straddle decision boundaries~\citep{peng_fingerprinting_2022}, or confidence levels of surrogate models~\citep{maini_dataset_2021}.
% But our problem setting does not have a classification task.
% Another approach is based on the idea that surrogate models are more confident about points in the victim's training set~\citep{maini_dataset_2021}.
For GNNs, existing methods rely on the similarity of embeddings between victim and surrogate models~\citep{grove}. 
This approach applies only when the victim and surrogate GNNs have the same embedding dimension, limiting its generality.
In contrast, our approach works for surrogates with any GNN architecture or embedding dimension.

\section{Proposed Method}
\label{sec:prop}

% A graph is a tuple $(G, X)$ where $G$ represents the links between $|G|$ nodes, each node~$i$ has features $\bx_i\in\mathbb{R}^D$, and $X=\begin{bmatrix}\bx_1, \ldots, \bx_{|G|}\end{bmatrix}$.
% We are given a GNN model $M$ (the ``victim'' model) that takes any $(G, X)$ as input and outputs a node embedding $\bh_i\in\mathbb{R}^d$ for each node~$i$.
% Each node's embedding depends on the entire graph, not just on that node's features, i.e., $\bh_i:=\bh_i(G, X; M)$.
% To simplify notation, we will refer to $\bh_i(G,X;M)$ as $\bh_i(X)$ when the context is clear.

% A graph is a tuple $(G, X)$ where $G = (V, E)$ represents the links between $n := |V|$ nodes, each node $i$ has features $\bx_i \in \mathbb{R}^D$, and $X = \bigl(\bx_1, \ldots, \bx_n\bigr)$. 
% Let $\mathcal{D}$ denote the data distribution over graph--feature pairs $(G, X)$.
% A pair $(G, X)$ is {\em feasible} if lies within the support of $\mathcal{D}$.
% A $d$-dimensional GNN $H$ is a function $\bh:(G,X)\to\mathbb{R}^{d\times n}$.
% The $i^{th}$ column of the matrix $\bh(G,X; H)$ is called the {\em embedding} of node~$i$ and is denoted by $\bh_i(G,X; H)\in\mathbb{R}^d$.
A graph is a tuple $(G, X)$ where $G = (V, E)$ has $n := |V|$ nodes, each with features
$\bx_i \in \mathbb{R}^D$, and $X = (\bx_1, \ldots, \bx_n)$. Let $\mathcal{D}$ denote the data
distribution over feasible pairs $(G, X)$. A $d$-dimensional GNN $H$ is a function
$\bh:(G,X)\to\mathbb{R}^{d\times n}$, whose $i^{th}$ column $\bh_i(G,X;H)\in\mathbb{R}^d$ (or $\bh_i(X)$ when
context is clear) is the {\em embedding} of node $i$.
We are given a $d$-dimensional {\em victim} GNN $M$ whose embeddings have the following property.

%with node embeddings $\bh_i(G, X; M)\in \mathbb{R}^d$

% To simplify notation, we will refer to $\bh_i(G, X; H)$ as $\bh_i(X)$ or $\bh_i$ when the context is clear.

% A $d$-dimensional GNN is a function that takes a feasible $(G,X)$ as input and outputs a $d\times n$ matrix whose $i^{th}$ column is a $d$-dimensional vector called the {\em embedding} of node $i\in \{1, \ldots, n\}$.

%To simplify notation, we will refer to $\bh_i(G, X; M)$ as $\bh_i(X)$ or $\bh_i$ when the context is clear.
%We assume the following.

% \txtred{
% We are given a GNN model $M$ (the ``victim'' model) that takes any $(G, X)$ as input and outputs a $d$-dimensional node embedding $h_i(G, X; M) \in \mathbb{R}^d$ for each node $i\in\{1, \ldots, n\}$. 
% To simplify notation, we will refer to $\bm{h_i}(G, X; M)$ as $\bm{h_i}(X)$ when the context is clear. 
% }

% Each node's embedding depends on the entire graph, not just on that node's features, i.e., $\bm{h_i} := \bm{h_i}(G, X; M)$. 

\begin{assumption}[Embedding properties]
    $\bh_i(X)$ is twice-differentiable with bounded Hessian, and there exist constants $\ualpha_M$ and $\oalpha_M$ such that 
    $0< \ualpha_M \leq \|\bh_i\| \leq \oalpha_M$ 
    for all feasible $G$, $X$, and~$i$.
    \label{A:bhi}
\end{assumption}

Next, we see a $d'$-dimensional GNN model $Z$, possibly using a different architecture and dimension than $M$.
We can query $Z$ with a few graphs and observe its outputs, but we have no access to $Z$'s internals.
Our goal is to {\em determine whether GNN $Z$ is an independently trained model or a surrogate of the victim GNN $M$.} We define a surrogate model as follows.

\begin{definition}[Surrogate model]
A $d'$-dimensional GNN~$M^\prime$, with embeddings $\bh_i^\prime:=\bh_i(G, X; M^\prime)\in\mathbb{R}^{d'}$, is a surrogate of GNN $M$ if {\bf there exist}
functions $\phi:\mathbb{R}^d\to\mathbb{R}^{d'}$ and $\hbh:(G,X)\to\mathbb{R}^{d\times n}$
%, and a small $\epsilon\geq 0$ (specified later)
such that
%, for every $(G,X)$,
%feasible $(G,X)$,
%a function $\phi:\mathbb{R}^d\to\mathbb{R}^{d'}$ and vectors $\hbh_i:=\hbh_i(G,X; M^\prime)\in\mathbb{R}^d$ for every feasible $(G,X,i)$ such that
\begin{align}
    \sup_{G, X, i} & \|\hbh_i(G, X) - \bh_i(G, X; M)\| \leq \epsilon, &
    \bh_i^\prime &= \phi\left(\hbh_i\right),
    \label{eq:epsilon}
\end{align}
for every graph $(G,X)$ that is feasible or is in a $\delta$-neighborhood of a feasible graph ($\epsilon$ and $\delta$ are small positive numbers specified later).
% where  $\epsilon\geq 0$ is small and specified later.
%Note that $M^\prime$ outputs $\bh_i^\prime$, not $\hbh_i$; we only assume that $\hbh_i$ exist.
\label{defn:surrogate}
%\vspace{-1em}
\end{definition}

For intuition, the adversary can query the victim model~$M$ on graphs $(G,X)$, collect the resulting embeddings $\bh_i(G,X;M)$ for all nodes, and train a GNN whose embeddings $\hat{\bh}_i$ approximate~$\bh_i$. The adversary may then apply an additional transformation layer that outputs~$\phi(\hat{\bh}_i)$.

Alternatively, the adversary can train $M'$ end-to-end using $(G,X)$ as input and $\phi(\bh_i)$ as the target output.
Even though the adversary never explicitly constructs an intermediate representation $\hat{\bh}_i$, it is still useful to characterize the connection between $\bh_i$ and $\bh_i^\prime$.

In both cases, the victim and surrogate models output embeddings $\bh_i \in \mathbb{R}^d$ and $\bh_i' \in \mathbb{R}^{d'}$, respectively, and our algorithm relies only on these outputs. The link between the two is given by the (unobserved) $\hat{\bh}_i$ and $\phi(\cdot)$ introduced in Definition~\ref{defn:surrogate}.
Their existence is used solely in our analysis.

\smallskip
\begin{remark}
The sup-norm in Equation~\ref{eq:epsilon} can be relaxed to the supremum over a subset of graphs $(G,X)$ with a high probability under the data distribution.
Furthermore, we only need Assumption~\ref{A:bhi} to hold at around our query tuples, defined later.
We keep this form to simplify the exposition.
%We keep the form of Equation~\ref{eq:epsilon} to simplify the exposition.
\end{remark}

Next, we motivate our proposed method on the special case of no reconstruction error ($\epsilon=0$) before discussing the general case.

\subsection{The Perfect Reconstruction Case}
\label{sec:prop:perfect}
Here, $\hbh_i=\bh_i$, so $M^\prime$ outputs embeddings $\bh_i^\prime=\phi(\bh_i)$.
Not all choices of $\phi(\cdot)$ are realistic.
For example, $\phi(\bh)=\bm{0}$ would lead to useless embeddings.
More generally, if $\phi(\bh_1)=\phi(\bh_2)$ for $\bh_1\neq\bh_2$, then $\phi(\cdot)$ loses information.
This can hurt performance on downstream tasks, so the adversary will avoid such a $\phi(\cdot)$.

% \begin{assumption}[No information loss]
% The transformation $\phi(\cdot)$ is differentiable and invertible, that is, for any unit vector $\bw\in\mathbb{R}^d, \|\bw\|=1$ and any $\bh\in\mathbb{R}^d$,
% $
%     \left.\nabla_{\bw}\phi\right|_\bh := \lim_{t\to 0} \frac{\phi(\bh+t\bw) - \phi(\bh)}{t} \neq \bm{0}.
% $
% \label{A:invertible}
% \end{assumption}

\begin{assumption}[Local invertibility]
The transformation $\phi(\cdot)$ is differentiable and locally invertible:
\begin{align}
    \left.\nabla_{\bw}\phi\right|_\bh \; &:=\; \lim_{t\to 0} \frac{\phi(\bh+t\bw) - \phi(\bh)}{t} \;\neq\; \bm{0} &\forall
\bh\in\mathbb{R}^d, \bw\in\mathbb{R}^d, \|\bw\|=1.
\end{align}
\label{A:invertible}
\end{assumption}

Assumption~\ref{A:invertible} ensures that $\phi(\bh_1)\neq\phi(\bh_2)$ for any two close embeddings $\bh_1\neq\bh_2$.
The assumption holds for many intuitive transformations, such as rotation, non-zero scaling, projection to a higher dimension, and translation, and for any composition of such functions (Lemma~\ref{lem:closure} in the Appendix).
Note that projecting to a lower dimension loses information, so the adversary will avoid $d'<d$.

% For instance, $\phi(\bh)=Qh$ with non-singular $Q^TQ$ satisfies the assumption.
% This includes rotations of $\bh$, non-zero scaling of the components of $\bh$, and projection to a higher dimension.
% Translations $(\phi(\bh)=\bh+\bm{c})$ also satisfy it.

% \begin{lemma}[Closure under composition]\label{lem:closure}
% If $f: \mathbb{R}^d \to \mathbb{R}^{d'}$ and $g: \mathbb{R}^{d''} \to \mathbb{R}^d$ each satisfy Assumption~\ref{A:invertible}, then $f \circ g$ also satisfies Assumption~\ref{A:invertible}.
% \end{lemma}
 
% \begin{proof}
% Let $w$ be any unit vector and $h$ any point in the domain of $g$.
% By the chain rule,
% \[
%   \nabla_w (f \circ g)\big|_h
%   = J_f\big|_{g(h)} \cdot J_g\big|_h \, w
%   = J_f\big|_{g(h)} \cdot \nabla_w g\big|_h.
% \]
% Since $g$ satisfies Assumption~\ref{A:invertible}, $\nabla_w g\big|_h \neq 0$.
% Let $v := \nabla_w g\big|_h$; then $v \neq 0$, so $v / \|v\|$ is a unit vector.
% Since $f$ satisfies Assumption~\ref{A:invertible},
% \[
%   J_f\big|_{g(h)} \, v
%   = \|v\| \cdot \nabla_{v/\|v\|} f\big|_{g(h)}
%   \neq 0.
% \]
% Hence $\nabla_w (f \circ g)\big|_h \neq 0$, so $f \circ g$ satisfies Assumption~\ref{A:invertible}.
% \end{proof}
% \txtblue{(RN: Moved proof to appendix)}

We need a fingerprint that is invariant under any transformation satisfying Assumption~\ref{A:invertible}.
We will show that this strong requirement is met by the stationary points of $\bh_i(X)$, as defined below.

\begin{definition}
Let $F$ be any GNN.
We define the {\em query tuples} $\mathcal{Q}$, {\em directional derivative} $\nabla_{\bw}\bh_i(G,X;F)$ of an embedding at a query tuple, and the {\em stationary points} $\mathcal{S}(F)$ of GNN~$F$:
    \begin{align}
    \mathcal{Q} &= \bigl\{ (G, X, i, \bw) \mid (G,X) \in \mathrm{supp}(\mathcal{D}),\; 
    i \in \{1,\ldots,n\},\; \bw \in \mathbb{R}^D,\; \|\bw\| = 1 \bigr\}\nonumber\\
    \bh_i^{(\tau\bw)}(G,X;F) &= \bh_i\left(G, \begin{bmatrix}\bx_1, \ldots, \bx_{i-1}, \bx_i+\tau\bw, \bx_{i+1}, \ldots, \bx_{n}\end{bmatrix}; F\right),\nonumber\\
    \nabla_{\bw}\bh_i(G,X;F) &= \lim_{\tau\to 0} \frac{\bh_i^{(\tau\bw)}(G,X;F) - \bh_i(G,X;F)}{\tau} \quad\forall (G,X,i,\bw)\in\mathcal{Q}\nonumber\\ %\label{eq:hdelta} \\    
    %&\quad\text{ and } (G,X,i,\bw)\in\mathcal{Q}\nonumber\\
    \mathcal{S}(F) &= \bigl\{ (G, X, i, \bw) \in \mathcal{Q} \mid \|\nabla_{\bw} \bh_i(G, X; F)\| = 0 \bigr\}.\nonumber %  \label{defn:stationary}
    \end{align}
    We will refer to $\nabla_{\bw}\bh_i(G,X;F)$ as $\nabla_{\bw}\bh_i(X)$ when the context is clear.
\label{defn:stationary}
\end{definition}

A stationary point is a graph $G$ with features $X$ such that nudging a node's features $\bx_i$ in the direction $\pm{\bw}$ does not change its embedding $\bh_i$ to first order.
But if $\bh_i$ is unaffected, so is $\phi(\bh_i)$, {\em for any $\phi(\cdot)$.}
Thus, this point is also stationary for the surrogate model.
The next lemma formalizes this property.

\begin{lemma}
Let $\bw\in\mathbb{R}^D, \|\bw\|=1$, and $\bh_i^\prime=\phi(\bh_i)$ for some $\phi(\cdot)$ satisfying Assumption~\ref{A:invertible}.
Then, $\|\nabla_{\bw}\bh_i(X)\|=\bm{0}$ if and only if $\|\nabla_{\bw}\bh_i^\prime(X)\|=\bm{0}$.
\label{lem:stat_invariant}
\end{lemma}

This suggests the following algorithm: given the victim GNN~$M$ and the unknown GNN~$Z$, test if $\mathcal{S}(M)=\mathcal{S}(Z)$.
By Lemma~\ref{lem:stat_invariant}, the test succeeds if $Z$ is a surrogate of~$M$.
But we expect the test to fail if $Z$ is an independent model, since every such model should converge to a different local optimum.
Now, we observe empirically that $\mathcal{S}(M)$ and $\mathcal{S}(Z)$ are large sets, making direct comparison difficult.
Hence, we use a sampling-based approach, as follows.

% By Lemma~\ref{lem:stat_invariant}, $\mathcal{S}(M)=\mathcal{S}(M^\prime)$ for any surrogate model $M^\prime$, {\em irrespective of the transform $\phi(\cdot)$ used in $M^\prime$.}
% So, for any model~$Z$, we can test if the points in $\mathcal{S}(M)$ are also stationary points of~$Z$.
% If $Z$ is a surrogate of~$M$, all the tests will succeed.
% However, $\mathcal{S}(M)\neq \mathcal{S}(Z)$ if $Z$ is an independent model, since every such model should converge to a different local optimum.
% This motivates the following approach.

% Now, we observe empirically that $\mathcal{S}(M)$ and $\mathcal{S}(Z)$ are large sets, making direct comparison difficult.
% Instead, we assume that we can sample from $\mathcal{S}(M)$.
% For simplicity, we also assume here that we can compute $Z$'s gradients.
% We will relax this assumption in Section~\ref{sec:prop:general}, where we also discuss how we sample from $\mathcal{S}(M)$.
% Next, we present our Algorithm~\ref{alg:special} for surrogate detection.
% %and formally state our assumption.

\begin{algorithm}
    \caption{DetectSurrogate (Special case)}
    \begin{algorithmic}[1]
    \item[{\bf Input:}] GNNs $M$ and $Z$
    \State Draw independent samples $T$ from $\mathcal{S}(M)$
    \State Check if all points in $T$ are stationary for $Z$
    \State \Return {\em Surrogate} if $T\subseteq \mathcal{S}(Z)$ else {\em Independent}
    \end{algorithmic}
    \label{alg:special}
\end{algorithm}

In Section~\ref{sec:prop:general}, we will show how we sample from $\mathcal{S}(M)$.
We will also remove the need to compute $Z$'s gradients.
Next, we formalize the assumption of differences between the stationary points of the victim and any independent model, and use it to prove the algorithm's correctness.

%, we assume that if we sample a stationary point from $\mathcal{S}(M)$, then with probability at least $\gamma_M$ this point is not also a stationary point for the independent model $I$.

% We will relax this assumption in Section~\ref{sec:prop:general}, where we also discuss how we sample from $\mathcal{S}(M)$.
%, and at least a $\gamma_M$ fraction of these stationary points are unique to $M$ and not stationary for any independent model.

\begin{assumption}[Idiosyncratic stationary points]
 (a) $\mathcal{S}(M) $ is equipped with a probability measure $\mu_M$ from which we can draw samples.
 (b) There exists $\gamma_M > 0$ such that for any independently trained model $I$, $\mu_M\!\bigl(\mathcal{S}(M) \setminus \mathcal{S}(I)\bigr) \;\geq\; \gamma_M.$
\label{A:idiosyncratic}
\end{assumption}
% Here, $\mu_M\!\bigl(\mathcal{S}(M) \setminus \mathcal{S}(I)\bigr)$ denotes the $\mu_M$-measure of the subset of stationary points of $M$ that are \emph{not} stationary points of $I$.

Informally, if we sample a stationary point from $\mathcal{S}(M)$ according to $\mu_M$, then with probability at least $\gamma_M$ this point is not also a stationary point for the independent model $I$.
% For instance, if $\mathcal{S}(M)$ were finite and $\mu_M$ were the uniform measure on $\mathcal{S}(M)$, then $\mu_M(\mathcal{S}(M) \setminus \mathcal{S}(I)) = |\mathcal{S}(M) \setminus \mathcal{S}(I)| / |\mathcal{S}(M)|$.
 
% Note that we do \emph{not} require $\mathcal{S}(M)$ to be countable.

% Assumption~\ref{A:idiosyncratic} was motivated by our exploratory analysis on real datasets, and the strong empirical performance of \ourmethod (Section~\ref{sec:exp}) suggests that it is appropriate.
% Using this, we prove the correctness of Algorithm~\ref{alg:special}.

% Next, we show that Algorithm~\ref{alg:special} is correct with high probability when the number of samples $|T|$ is sufficiently large.
% First, we formally state our assumptions.
\begin{theorem}
    If $\epsilon=0$ and Assumptions~\ref{A:bhi},~\ref{A:invertible}, and~\ref{A:idiosyncratic} hold, then Algorithm~\ref{alg:special} is correct with probability at least $1-e^{-2|T|\gamma_M^2}$.
    \label{thm:special}
\end{theorem}

\subsection{General Case}
\label{sec:prop:general}

%Here, the surrogate model $M^\prime$ outputs node embeddings $\phi(\hbh_i)$ where $\hbh_i$ is a reconstruction of the embedding $\bh_i$ from the victim model $M$.
Now, $\hbh_i(X)$ can be any function within a band of width $\epsilon$ around $\bh_i(X)$ (Definition~\ref{defn:surrogate}).
Hence, its stationary points may differ from $M$.
The function $\hbh_i(X)$ could even be discontinuous, so we cannot assume that its directional derivatives exist.
So we cannot directly check for stationary points in Algorithm~\ref{alg:special}.
Instead, we approximate the stationarity checks, as shown below.

\begin{definition}
Let $F$ be any GNN model.
For a tuple $t=(G, X, i, \bw) \in \mathcal{Q}$ and $\delta>0$, define,
\begin{align}
    q_F(t) &= \frac{\|\bh_i^{(\delta\bw)}(X) - \bh_i(X)\|}{\|\bh_i(X)\|}, &
    \beta_F &= \frac{E_{t\sim\mu_M} q_F(t)}{E_{t\sim\mathcal{D}_{exp}} q_F(t)},
     \label{eq:betaz}
    %\label{eq:q}
\end{align}
where $\bh_i^{(\delta\bw)}(X)$ depends on $t$ as defined in Definition~\ref{defn:stationary},
\begin{align*}
  (G, X) &\sim \mathcal{D}, &
  i \mid G, X &\sim \mathrm{Uniform}(1, \ldots, n),&
  \bw \mid G, X, i &\sim \mathrm{Uniform}(\{v \in \mathbb{R}^D,\, \|v\| = 1\}).
\end{align*}
\label{defn:randomchoice}
\end{definition}
%the embedding of node~$i$ after a $\delta$-change in its features in the direction $\bw$

% We define the distribution $\mathcal{D}_{\mathrm{exp}}$ over the query tuples $\mathcal{Q}$ (Definition~\ref{defn:stationary}), and the ratio $\beta_F$ for any model~$F$ as follows:
% \begin{align*}
%   (G, X) &\sim \mathcal{D}, &
%   i \mid G, X &\sim \mathrm{Uniform}(1, \ldots, n),&
%   \bw \mid G, X, i &\sim \mathrm{Uniform}(\{v \in \mathbb{R}^D,\, \|v\| = 1\}),
% \end{align*}
% \label{defn:randomchoice}
% \end{definition}
% Using Definition~\ref{defn:randomchoice}, we define the ratio
% \begin{align}
%     \beta_F := \frac{E_{t\sim\mu_M} q_F(t)}{E_{t\sim\mathcal{D}_{exp}} q_F(t)},
%     \label{eq:betaz}
% \end{align}
% where $F$ is any model and $\mu_M$ is the probability measure over $\mathcal{S}(M) $ (Assumption~\ref{A:idiosyncratic}).

% Let $F$ be any GNN model.
% For a tuple $t=(G, X, i, \bw) \in \mathcal{Q}$ and $\delta>0$, define,
% \begin{align}
%     q_F(t) := \frac{\|\bh_i^{(\delta\bw)}(X) - \bh_i(X)\|}{\|\bh_i(X)\|},
%     \label{eq:q}
% \end{align}
% where $\bh_i^{(\delta\bw)}(X)$ is the embedding of node~$i$ after a $\delta$-change in its features in the direction $\bw$ (see Equation~\ref{eq:hdelta}).
% All terms in the RHS are functions of the model $F$, graph $G$, features $X$, and node $i$ from the tuple $t$. 

% \txtred{Here, $t = (G, X, i, \bw)$ is a tuple consisting of a graph $G$ with node features $X$, a node index $i$, and a direction $w$.}
For intuition, $q_F(t)$ is a (normalized) magnitude of change in an embedding when we change the features $\bx_i$ of node~$i$ in the direction $\bw$, and $\beta_F$ compares the expected change near the stationary points of~$M$ versus randomly chosen points in~$\mathcal{Q}$.
The formula for~$\beta_F$ is also invariant under common transformations of the embeddings (Theorem~\ref{thm:betaZ}).
We will show below that $\beta_F$ is small when $F$ is the victim model or a surrogate, but not when $F$ is an independent model.
Thus, $\beta_F$ mimics the properties of the directional derivative at stationary points $\mathcal{S}(M)$.

\begin{lemma}
    Under Assumption~\ref{A:bhi}, for $\delta$ small enough, $\beta_M=O(\delta)$.
    \label{lem:betam}
\end{lemma}

Lemma~\ref{lem:betam} shows that $\beta_M$ is small for the victim model.
Next, consider $\beta_{M^\prime}$ for a surrogate model $M^\prime$.
Here, we need a stronger version of Assumption~\ref{A:invertible}.
Instead of just requiring $\|\phi(\bv_1)-\phi(\bv_2)\|\neq 0$ when $\|\bv_1-\bv_2\|\neq 0$, we now bound the change in norm due to $\phi(\cdot)$.
%that quantifies the maximum change between two embeddings when transformed by $\phi(\cdot)$.

% Ideally, we want $\beta_{M^\prime}=\beta_M$.
% The next theorem shows that this is indeed the case for some common transformations when $\epsilon=0$.
% When $\epsilon>0$, $\beta_{M^\prime}$ depends on $\epsilon, \delta,$ and the characteristics of $\phi(\cdot)$.
%For a worst-case analysis, 

\begin{assumption}\label{A:phistrong}
$\phi(\bm{0})=\bm{0}$, and there exist positive constants $c, C$ such that
\begin{equation}
    c\|\bv_1-\bv_2\| \;\leq\; \|\phi(\bv_1)-\phi(\bv_2)\| \;\leq\; C\|\bv_1-\bv_2\|.
\end{equation}
\end{assumption}

\begin{lemma}\label{lem:betamprime}
Suppose Assumptions~\ref{A:bhi} and~\ref{A:phistrong} hold, and $\epsilon$ is as defined in
Definition~\ref{defn:surrogate}. Then, for $\delta$ small enough and any surrogate $M^\prime$ with $\epsilon=o(\delta)$,
\begin{equation}
  \beta_{M^\prime} \;=\; O\!\left((C/c)^2\cdot\max(\epsilon/\delta,\, \delta)\right).
\end{equation}
\end{lemma}

The adversary wants a low reconstruction error $\epsilon$ so that the downstream accuracy of the surrogate $M^\prime$ is comparable to that of the victim $M$.
But, by Lemma~\ref{lem:betamprime}, a small $\epsilon$ forces $M^\prime$ to have a small $\beta_{M^\prime}$, just like $\beta_M$ (Lemma~\ref{lem:betam}).
This result is independent of how the adversary trains $M^\prime$. Next, consider an independent model~$I$.
By Assumption~\ref{A:idiosyncratic}, at least a $\gamma_M$ fraction of the stationary points of $M$ are unique to $M$, so they are not stationary for~$I$.
We further qualify this assumption next.

\begin{assumption}\label{A:q_idio}
Idiosyncratic stationary points of $M$ are not special for independent models $I$:
\begin{equation}
    E_{t\sim\mu_M}\!\left[ q_I(t)\mid t\notin \mathcal{S}(I)\right] \;=\; E_{t\sim\mathcal{D}_{\mathrm{exp}}}\!\left[ q_I(t) \right].
\end{equation}
\end{assumption}

Assumption~\ref{A:q_idio} states that the idiosyncratic stationary points of $M$ are not special for~$I$ in any way.
We only require the independent model~$I$ to deviate significantly from $M$ for these points; however, we keep this version to simplify the exposition.
Figure~\ref{fig:indep_psc} in the Appendix shows some evidence in favor of this assumption.
With this assumption, we show that $\beta_I=O(1)$, unlike surrogate models.

\smallskip
% Next, we show that $\beta_I$ is $O(1)$.
%for independent models, as shown next.
%In other words, the expected value of $q_I(t)$ for these points $t\in\mathcal{S}(M)\setminus \mathcal{S}(I)$ is the same as that for $t\sim\mathcal{D}_{exp}$.

\begin{lemma}
    Let $M$ and $I$ be independent models.
    Suppose Assumptions~\ref{A:bhi} holds for both models, and Assumptions~\ref{A:idiosyncratic} and~\ref{A:q_idio} hold.
    Then, for $\delta$ small enough, $\beta_I\geq \gamma_M + O(\delta),$ where $\gamma_M$ is defined in Assumption~\ref{A:idiosyncratic}.
    \label{lem:betai}
\end{lemma}

Our goal is to determine if the GNN $Z$ is a surrogate of the victim $M$ or an independent model.
Lemmas~\ref{lem:betam},~\ref{lem:betamprime}, and~\ref{lem:betai} show that $\beta_Z$ is small ($O(\delta)$ or $O(\epsilon/\delta+\delta)$) if $Z$ is the victim model or a surrogate, but at least $\gamma_M$ if $Z$ is independent.
This motivates Algorithm~\ref{alg:general} (also see Remark~\ref{rem:actual_betaZ} in Appendix~\ref{sec:app:proofs}).

% \begin{algorithm}
%     \caption{DetectSurrogate (General case)}
%     \begin{algorithmic}[1]
%     \item[{\bf Input:}] GNNs $M$ and $Z$, number of samples $\ell$, threshold $\theta$
%     \State $T\gets \ell \text{ independent samples from }\mathcal{S}(M)$ \Comment{(Def.~\ref{defn:stationary})}
%     \State $R\gets \ell \text{ independent samples from }\mathcal{D}_{exp}$ over $\mathcal{Q}$ \Comment{(Def.~\ref{defn:randomchoice})}
%     \State $\hat{\beta}_Z\gets \frac{\sum_{t\in T} q_Z(t)}{\sum_{t\in R} q_Z(t)}$ \Comment{(Eqs.~\ref{eq:q},~\ref{eq:betaz})}
%     \State \Return {\em Surrogate} if $\hat{\beta}_Z\leq \theta$ else {\em Independent}
%     \end{algorithmic}
%     \label{alg:general}
% \end{algorithm}
\begin{algorithm}[h]
    \caption{DetectSurrogate (General case)}
    \label{alg:general}
    \small
    \setlength{\belowcaptionskip}{-0.5em}
    \begin{algorithmic}[1]
    \item[{\bf Input:}] GNNs $M$ and $Z$, number of samples $\ell$, threshold $\theta$
    \State $T \gets \ell$ samples from $\mathcal{S}(M)$ \hfill\Comment{Def.~\ref{defn:stationary}}
    \State $R \gets \ell$ samples from $\mathcal{D}_{\exp}$ over $\mathcal{Q}$ \hfill\Comment{Def.~\ref{defn:randomchoice}}
    \State $\hat{\beta}_Z \gets \big(\sum_{t\in T} q_Z(t)\big) / \big(\sum_{t\in R} q_Z(t)\big)$ \hfill\Comment{Eq.~\ref{eq:betaz}}
    \State \Return {\em Surrogate} if $\hat{\beta}_Z \leq \theta$ else {\em Independent}
    \end{algorithmic}
\end{algorithm}

\begin{theorem}
    Under the assumptions of Lemmas~\ref{lem:betamprime} and~\ref{lem:betai} and large enough $\ell$, Algorithm~\ref{alg:general} with $\theta=\gamma_M/2$ is correct with high probability.
    \label{thm:general}
\end{theorem}

\textbf{Selecting $\theta$ in Algorithm~\ref{alg:general}:}
Theorem~\ref{thm:general} shows that $\theta$ depends on the unknown parameter $\gamma_M$.
However, $\gamma_M$ is upper-bounded by $\beta_I$ for any independent model~$I$ (Lemma~\ref{lem:betai}).
Thus, we can train several independent GNNs on the same data as $M$ to get an upper bound for $\gamma_M$, which we then use to set the threshold.
Alternatively, one can use the distribution of $\beta_I$ over all datasets and GNN architectures (Figure~\ref{fig:indep_psc} in the Appendix).

%\vspace{-1em}
\textbf{Sampling stationary points:}
Recall that a stationary point is a tuple $(G, X, i, \bw) \in \mathcal{Q}$ such that $\|\nabla_{\bw} \bh_i(G, x; M)\|=0$ (Definition~\ref{defn:stationary}).
To find such a point, we sample a point $(G, X_0, i, \bw)\sim\mathcal{D}_{exp}$ (Definition~\ref{defn:randomchoice})  and solve
\begin{equation}
    X = \argmin_X \left(\frac{\|\nabla_{\bw} \bh_i(G, X; M)\|}{\|\bh_i(G, X; M)\|} + \lambda\cdot \frac{\|X-X_0\|_F}{\|X_0\|_F}\right).
    \label{eq:sample}
\end{equation}
We then use $(G, X, i, \bw)$ when a sample from $\mathcal{S}(M)$ is needed in Algorithm~\ref{alg:general}.
In the first term, we normalize $\|\nabla_{\bw} \bh_i\|$ by $\|\bh_i\|$ since this matches the form of $q_M$ (Definition~\ref{defn:randomchoice}).
The second term ensures that the chosen node features $X$ are close to $X_0$ drawn from the data distribution.
Hence, $(G, X)$ looks like a realistic graph.
We use a solver that does not require access to gradients~\citep{bennet2021nevergrad}.
Figure~\ref{fig:cosine} in Appendix~\ref{sec:app:details} shows that the graphs picked using Equation~\ref{eq:sample} are all distinct. % from each other.

The optimization in Eq.~\ref{eq:sample} operates on the node features $X$, not on the graph structure $G$.
For very large graphs, it is not necessary to optimize over the entire feature matrix $X$.
Since the goal is to find a stationary point for the embedding of a single node $i$, it suffices to optimize over the features of node $i$ and its $k$-hop neighborhood.
This reduces the per-iteration cost to a forward pass over a local subgraph of size depending on $k$, rather than the full graph.
In our experiments, the embedding dimension has a larger effect on runtime than graph size.

\textbf{Extension to integer-valued features:}
When the features are constrained to be integers, we cannot construct $\bh_i^{(\delta\bw)}$ (Definition~\ref{defn:stationary}) for arbitrary $\delta>0$ and direction~$\bw$.
Instead, we set $\delta=1$ and sample $\bw\in\mathbb{R}^D$ from the set of axis-aligned unit vectors:
$
    \bw = \bm{e}_j, \text{ where } j\sim \text{Uniform}\{1, \ldots, D\},
    % \bw\sim \text{Uniform}\{\underbrace{\begin{bmatrix}0\ldots 0 1 0\ldots 0\end{bmatrix}}_{\text{only 1 for axis $j$}}\mid j\in\{1,\ldots,D\}\}.
$
and $\bm{e}_j$ is the standard basis-vector along the $j^{th}$ feature.
%With this change, $q_F(t)$ measures the relative change in a node's embedding when we increment one component of that node's features.
We also replace $\nabla_{\bw} \bh_i$ by $\bh_i^{(\bw)}-\bh_i$ in Equation~\ref{eq:sample} when sampling stationary points.
While Theorem~\ref{thm:general} assumes a small $\delta$, we show empirically that this approach works well in practice.

% In the formula for $q_F(t)$ (Equation~\ref{eq:q}) needed for $\beta_Z$, we can no longer construct $\bh_i^{(\delta\bw)}$ (Equation~\ref{eq:hdelta}) for arbitrary $\delta>0$ and direction~$\bw$ .

% In other words, $q_Z(t)$ measures the relative change in a node's embedding when we increment one component of that node's features.
% We also replace $\nabla_{\bw} \bh_i$ by $\bh_i^{(\bw)}-\bh_i$ in Equation~\ref{eq:sample} when sampling stationary points.

\textbf{Robustness against countermeasures:}
Suppose the adversary can examine the queries sent to their surrogate model.
The adversary cannot detect or penalize queries about stationary points for two reasons.
First, the regularization in Equation~\ref{eq:sample} makes those graphs look realistic.
Second, recall that we query the surrogate with a graph $(G, X)$.
Our chosen $(G, X)$ is stationary only for a random node~$i$ and direction~$\bw$, both unknown to the adversary.

The adversary could take a different approach by adding noise to their output embeddings.
However, this degrades their downstream performance.
Furthermore, the adversary does not know how much noise is needed to go beyond \ourmethod’s surrogate detection threshold.
This depends on \ourmethod's scores for independent models, which the adversary does not have access to.
We also show in Section~\ref{sec:exp} that \ourmethod is robust against adversarial attacks such as pruning and fine-tuning, which introduce noise.

Finally, note that Algorithm~\ref{alg:general} generates new query graphs each time.
Thus, even if some query graphs leak, future surrogate detection remains secure.

% \begin{remark}
%     In Algorithm~\ref{alg:general}, we can use an alternative formula for $\hat{\beta}_Z$:
%     \begin{align}
%         \hat{\beta}_Z &= \text{Mean}\left(\{ \text{percentile of $q_Z(t)$ in $U$}\}_{t\in T}\right)\nonumber\\
%         \text{where }U &= \{q_Z(t)\mid t\in R\}.\label{eq:percent}
%     \end{align}
%     % \txtblue{RN: Will it be confusing to use $Q$ here, but $\mathcal{Q}$ for the query tuple set? DC: Changed to U?}
    
%     Here, $U$ contains samples of $q_Z(t)$ for $t\sim\mathcal{D}_{exp}$, supported on the positive real line.
%     Now, consider a point $t\in T$.
%     If $t$ is a stationary point for $Z$, $q_Z(t)\approx 0$, so the percentile of $q_Z(t)$ in $U$ is also small.
%    Otherwise, $t$ is an idiosyncratic stationary point of $M$.
%    In this situation, we expect $q_Z(t)$ to be distributed like a tuple drawn from $\mathcal{D}_{exp}$, i.e., with a percentile score distributed uniformly between $0$ and $100$.
%     Thus, $\hat{\beta}_Z$ is small if $Z$ is a surrogate but high if $Z$ is independent.
%    This mirrors the behavior of $\beta_Z$ (Equation~\ref{eq:betaz}) while avoiding division-related instability.   
%    We use this variant in our experiments.
% \end{remark}

\section{Experiments}
\label{sec:exp}

\begin{table*}[t]
\centering
\scriptsize
\begin{tabular}{l|ccccc|ccccc}
\toprule
& \multicolumn{5}{c|}{\textbf{\ourmethod}} & \multicolumn{5}{c}{\textbf{PreGIP}}\\
\textbf{Dataset} 
    & \textbf{GCN} & \textbf{GIN} & \textbf{GSage} & \textbf{ARMA} & \textbf{MixHop}
    & \textbf{GCN} & \textbf{GIN} & \textbf{GSage} & \textbf{ARMA} & \textbf{MixHop} \\
\midrule
Citeseer
         & 1.00 & 1.00 & 1.00 & 1.00 & 1.00
         & 0.75 & \hl{0.00} & 0.90 & 0.85 & 0.90
\\
OGBMag
         & 1.00 & 1.00 & 1.00 & 1.00 & 1.00
         & 1.00 & \hl{0.30} & 1.00 & 1.00 & 0.95
\\
HIV
         & 1.00 & 1.00 & 0.64 & 0.78 & 1.00
         & 0.65 & 1.00 & 1.00 & 0.50 & 1.00
\\
Yelp
         & 1.00 & 1.00 & 1.00 & 0.94 & 1.00
         & 0.65 & 1.00 & 0.55 & 1.00 & 0.60
\\
MNIST
         & 1.00 & 1.00 & 1.00 & 0.78 & 1.00
         & 0.80 & 0.95 & 0.55 & 0.80 & 0.60
\\
BBBP
         & 1.00 & 1.00 & 0.93 & 1.00 & 1.00
         & 0.60 & \hl{0.00} & 0.70 & \hl{0.20} & 0.55
\\
Pubmed
         & 1.00 & 1.00 & 0.79 & 0.83 & 0.93
         & 0.80 & 1.00 & \hl{0.40} & 0.70 & 0.90
\\
QM9
         & 1.00 & 1.00 & 1.00 & 0.78 & 1.00
         & 1.00 & 1.00 & 1.00 & 0.90 & 1.00
\\
Fin
         & 1.00 & 0.86 & 0.86 & 0.78 & 1.00
         & 1.00 & 1.00 & 1.00 & 1.00 & 1.00
\\
Amazon
         & 1.00 & 0.64 & 1.00 & 0.67 & 1.00
         & 0.80 & 1.00 & 0.80 & 0.80 & 0.90
\\
Coco
         & 1.00 & 1.00 & 0.93 & 0.94 & 0.93
         & 1.00 & \hl{0.25} & 1.00 & 0.60 & \hl{0.30}
\\
DBLP
         & 1.00 & 1.00 & 0.64 & 1.00 & 1.00
         & 0.90 & \hl{0.00} & 1.00 & 0.95 & 0.90
\\
CIFAR
         & 1.00 & 1.00 & 1.00 & 0.94 & 1.00
         & 1.00 & 0.65 & 0.60 & \hl{0.30} & 0.70
\\
Computers
         & 1.00 & 1.00 & 1.00 & 1.00 & 1.00
         & 1.00 & 1.00 & \hl{0.35} & 1.00 & 1.00
\\
\midrule
{\bf Average}
         & {\bf 1.00} & {\bf 0.96} & {\bf 0.91} & {\bf 0.89} & {\bf 0.99}
         & 0.85 & 0.65 & 0.77 & 0.76 & 0.81
\\

\bottomrule
\end{tabular}
\caption{\textbf{ AUC of surrogate detection under model extraction attack (higher is better):}
Each AUC is computed by classifying all surrogates (across all random seeds and all 5 GNN architectures) against all independent models (again across seeds and architectures).
PreGIP's predictions are occasionally worse than random (in red).
Averaged over 14 datasets, \ourmethod dominates for all GNN architectures.
It has a perfect AUC for GCN, and is nearly perfect for MixHop.
%Instances where a method is worse than random are highlighted in red.
}
\label{tab:dataset-model-results}
\vspace{-1em}
\end{table*}

We ran experiments to assess the accuracy of \ourmethod and competing methods for the surrogate GNN identification problem.
We used 14 datasets from diverse domains and tested 5 popular GNN architectures against various adversarial attacks and embedding transformations.

% \subsection{Experimental Setup}
% \label{sec:exp:setup}

\textbf{Datasets:} 
We present results on 14 graph datasets covering molecules, citations, co-purchases, social, and financial networks.
Each dataset consists of several graphs with associated node features.
For datasets containing only a single graph, we used k-hop subgraphs around randomly chosen nodes.
All datasets are available from the {\em PyG} library, and their details are in Appendix~\ref{sec:app:details}.

\textbf{Base GNN Models:}
We used the GCN~\citep{kipf2017semi}, GIN~\citep{xu2019how}, GraphSage~\citep{hamilton2017inductive}, ARMA~\citep{bianchi2022conv}, and MixHop~\citep{abu-el-haija_mixhop_2019} GNN architectures.
For each dataset, we trained all models on two data splits per task (node or graph classification/regression).
For each model, we set it as the victim and the others as independent models.
The surrogate models are trained via the following adversarial attacks.

% \textbf{Adversarial attacks:}
% Our attacks are modeled on those from~\citep{grove}.
% In {\em model extraction,} we sampled graphs, got their embeddings from the victim model, and trained a GNN to output similar embeddings.
% We trained models for all five architectures and selected the one with the most similar embeddings to the victim's embeddings as the surrogate model.
% In {\em fine-tuning,} we fine-tuned a previously trained surrogate with training labels.
% In {\em pruning,} we randomly set some model weights to zero.

% \textbf{Transformations:}
% We consider the following functions:\\
% {\em Matrix multiplication} of the embeddings by a random orthogonal, Gaussian, or permutation matrix;\\
% {\em Scaling and translating} the embeddings by different levels for each vector component;\\
% {\em Projecting} embeddings to a higher dimension;\\
% Applying {\em component-wise functions} such as sigmoid; and\\
% {\em Normalizing} the embeddings.\\
% We also used combinations of these transformations.

\textbf{Competing methods:}
We compare our method with the recent watermarking method PreGIP~\citep{pregip} and the fingerprinting method GrOVe~\citep{grove}.
PreGIP's surrogate detection does not directly test embeddings.
So, for PreGIP, we score every graph by computing $\text{mean-distance}(G_a, G_b)$ over watermark graph pairs $(G_a, G_b)$, and normalizing it by the same formula for randomly-chosen pairs of graphs.
We compute surrogate detection AUCs using these scores.
Using other formulas, such as cosine similarity, reduces PreGIP's performance (Appendix~\ref{sec:app:details}).

\subsection{Accuracy under Model Extraction Attack}
\label{sec:exp:modelExtract}

% In model extraction, the adversary uses the victim model to construct a training set.
% The adversary then trains a surrogate to accurately replicate the victim model's embeddings.
% Therefore, surrogate models generated via model extraction are likely to perform similarly to the victim model on downstream tasks.

%We compare \ourmethod to watermarking and show it outperforms watermarks under model extraction attacks.
We first show that watermarking via PreGIP is vulnerable to model extraction attacks.
Here, the adversary samples graphs, gets their embeddings from the victim model, and trains a GNN to output similar embeddings.
For each dataset, we trained independent models using multiple GNN architectures across two data splits. Treating each as a victim, we then trained surrogate models. For PreGIP, we followed the same procedure, but watermarked the victim before training the surrogate.

We evaluated both \ourmethod and PreGIP on all surrogate and independent models. Each method assigns a score to determine whether a model is a surrogate or independent. Table~\ref{tab:dataset-model-results} reports the AUC of surrogate detection based on these scores.
%We make the following observations.

\textbf{\ourmethod is accurate:}
\ourmethod has {\bf perfect surrogate detection (AUC=1)} on $73\%$ of all combinations of datasets and GNN architectures.
The AUC exceeds $0.9$ in $83\%$ of the cases.
\ourmethod's average AUC is highest for GCN (always perfect), MixHop ($0.99$), and GIN ($0.96$).

\textbf{\ourmethod outperforms watermarks:}
Across all GNN architectures, {\bf \ourmethod's AUC is 17\%-48\% better} than PreGIP.
For GIN, \ourmethod outperforms by $48\%$.
Furthermore, PreGIP's scoring is sometimes worse than random (highlighted in red in Table~\ref{tab:dataset-model-results}).

Recall that in model extraction, the adversary uses the victim model to create a new training set.
This set is unlikely to contain PreGIP's watermark graphs.
This limits the effectiveness of watermarking.
In contrast, \ourmethod uses no watermarks and works well for all architectures.

\subsection{Accuracy under Embedding Transformations}
Having shown that watermarking (PreGIP) is vulnerable to model extraction, we now examine the robustness of fingerprinting (GrOVe) under embedding transformations. We constructed surrogate models by applying various transformations to their output embeddings, including ones that may violate our assumptions, and evaluated all methods on these modified models.

Recall that, given an independent model and a surrogate, we compared their scores to determine which is the surrogate. After transformation, the surrogate’s score may change, potentially altering this comparison. To quantify this effect, we measured the fraction of independent models for which the prediction differed before and after transformation. A lower fraction indicates greater robustness. Table~\ref{tab:advTransforms} reports these fractions, averaged over all 14 datasets.

% Having shown that watermarking (PreGIP) is susceptible to model extraction, we now show that fingerprinting (GrOVe) suffers under embedding transformation.
% We created new surrogate models by transforming their output embeddings with various functions, which do not necessarily obey all our assumptions. We then evaluated transformed models using all methods. Recall that, given an independent model and a surrogate, we compare their scores to determine which is the surrogate.
% Now, the original and transformed surrogates can have different scores.
% As a result, comparing the two against the same independent model can yield different predictions.
% To quantify this effect, we recorded the fraction of independent models where the prediction changed after transformation.
% A smaller fraction indicates that a method is more robust to the transformation.
% Table~\ref{tab:advTransforms} lists these fractions, averaged over all 14 datasets.

\textbf{\ourmethod is robust against embedding transformations:}
\ourmethod is immune to rotation, scaling, projection to a higher dimension, and combinations of these ($0\%$ change).
This matches Theorem~\ref{thm:betaZ}.
We find that exponentiation has the greatest effect on \ourmethod.
%Specifically, translations affect the normalization in Equation~\ref{eq:q}, while 
Exponentiation distorts distances between embeddings (Assumption~\ref{A:phistrong}), which affects detection accuracy (Lemma~\ref{lem:betamprime}).
We note that \ourmethod works for translations, even though they do not match Assumption~\ref{A:phistrong}.
%We note that the effect of translations can be removed by shifting the origin to center the embeddings.
%Overall, \ourmethod is more robust to embedding transformations than any competing method.

\textbf{PreGIP and GrOVe are affected in different ways:}
PreGIP is much more susceptible to exponential and power transformations than \ourmethod.
In contrast, GrOVe is inapplicable when the surrogate embeddings are projected to higher dimensions.
GrOVe is also more sensitive than \ourmethod to the scaling of embedding components.

\begin{table*}[t]
\centering
\scriptsize
\resizebox{\textwidth}{!}{
\begin{tabular}{l|c@{\hspace{0.5em}}c@{\hspace{0.5em}}c@{\hspace{0.5em}}c@{\hspace{0.5em}}c|c@{\hspace{0.5em}}c@{\hspace{0.5em}}c@{\hspace{0.5em}}c@{\hspace{0.5em}}c|c@{\hspace{0.5em}}c@{\hspace{0.5em}}c@{\hspace{0.5em}}c@{\hspace{0.5em}}c}
\toprule
& \multicolumn{5}{c|}{\bf \ourmethod}
& \multicolumn{5}{c|}{\bf PreGIP}
& \multicolumn{5}{c}{\bf GrOVe}\\
{\bf Embedding Transformation} & \textbf{GCN} & \textbf{GIN} & \textbf{GSage} & \textbf{ARMA} & \textbf{MixHop} 
 & \textbf{GCN} & \textbf{GIN} & \textbf{GSage} & \textbf{ARMA} & \textbf{MixHop} 
 & \textbf{GCN} & \textbf{GIN} & \textbf{GSage} & \textbf{ARMA} & \textbf{MixHop} \\
\midrule

Permute
         & 0\% & 0\% & 0\% & 0\% & 0\%
         & 0\% & 0\% & 0\% & 0\% & 0\%
         & 0\% & 4\% & 4\% & 15\% & 0\%
\\
Rotate
         & 0\% & 0\% & 0\% & 0\% & 0\%
         & 0\% & 0\% & 0\% & 0\% & 0\%
         & 0\% & 4\% & 10\% & 16\% & 0\%
\\
Rotate, scale by 5
         & 0\% & 0\% & 0\% & 0\% & 0\%
         & 0\% & 0\% & 0\% & 0\% & 0\%
         & 0\% & 32\% & 28\% & 48\% & 0\%
\\
Project($\mathbb{R}^d \to \mathbb{R}^{5d}$), rotate, scale by 5
         & 0\% & 0\% & 0\% & 0\% & 0\%
         & 0\% & 0\% & 0\% & 0\% & 0\%
         & \hl{N/A} & \hl{N/A} & \hl{N/A} & \hl{N/A} & \hl{N/A}
\\
Multiply by $d\times d$ Gaussian mat.
         & 1\% & 1\% & 2\% & 2\% & 2\%
         & 2\% & 1\% & 4\% & 7\% & 3\%
         & 0\% & 18\% & 29\% & 35\% & 0\%
\\
(as above) and scale each entry by Unif$(1,10)$
         & 2\% & 1\% & 2\% & 3\% & 3\%
         & 3\% & 2\% & 3\% & 7\% & 4\%
         & 7\% & \hl{52\%} & \hl{50\%} & \hl{53\%} & 8\%
\\
Multiply by $5d\times d$ Gaussian mat.
         & 0\% & 0\% & 1\% & 1\% & 1\%
         & 1\% & 0\% & 2\% & 4\% & 2\%
         & \hl{N/A} & \hl{N/A} & \hl{N/A} & \hl{N/A} & \hl{N/A}
\\
(as above) and scale each entry by Unif$(1,10)$
         & 1\% & 1\% & 1\% & 1\% & 1\%
         & 1\% & 0\% & 2\% & 4\% & 3\%
         & \hl{N/A} & \hl{N/A} & \hl{N/A} & \hl{N/A} & \hl{N/A}
\\
\midrule
$\bh_{ij}\to \tan^{-1}(\bh_{ij})$
         & 3\% & 0\% & 4\% & 1\% & 2\%
         & 4\% & 45\% & 12\% & 3\% & 10\%
         & 0\% & 0\% & 0\% & 1\% & 0\%
\\
$\bh_{ij}\to \exp(\bh_{ij})$
         & 13\% & 9\% & 20\% & 40\% & 26\%
         & 11\% & \hl{77\%} & 20\% & 40\% & 31\%
         & 0\% & 7\% & 27\% & 8\% & 7\%
\\
$\bh_{ij}\to \text{sigmoid}(\bh_{ij})$
         & 7\% & 2\% & 14\% & 21\% & 12\%
         & 4\% & 43\% & 2\% & 3\% & 9\%
         & 0\% & 0\% & 0\% & 0\% & 0\%
\\
$\bh_{ij}\to \text{sinh}(\bh_{ij})$
         & 9\% & 1\% & 4\% & 16\% & 7\%
         & 9\% & \hl{77\%} & 14\% & 30\% & 34\%
         & 0\% & 6\% & 12\% & 7\% & 4\%
\\
$\bh_{ij}\to \text{tanh}(\bh_{ij})$
         & 3\% & 0\% & 5\% & 1\% & 2\%
         & 10\% & 48\% & 15\% & 8\% & 15\%
         & 0\% & 0\% & 0\% & 1\% & 0\%
\\
$\bh_{ij}\to\bh_{ij}^3$
         & 7\% & 5\% & 17\% & 16\% & 12\%
         & 9\% & \hl{77\%} & 44\% & \hl{66\%} & 34\%
         & 0\% & 8\% & 21\% & 10\% & 4\%
\\
$\bh_{ij}\to\bh_{ij}^5$
         & 7\% & 5\% & 21\% & 21\% & 16\%
         & 16\% & \hl{77\%} & \hl{54\%} & \hl{72\%} & 44\%
         & 4\% & 8\% & 25\% & 14\% & 7\%
\\
% $\bh_{ij}\to \bh_{ij}+1$
Translate each entry by $1$
         & 8\% & 2\% & 14\% & 21\% & 12\%
         & 0\% & 0\% & 0\% & 0\% & 0\%
         & 0\% & 0\% & 0\% & 0\% & 0\%
\\
%$\bh_{ij}\to \bh_{ij}+\mathrm{Unif}(1,5)$
Translate each entry by $\mathrm{Unif}(1,5)$
         & 8\% & 3\% & 18\% & 26\% & 18\%
         & 0\% & 0\% & 0\% & 0\% & 0\%
         & 0\% & 0\% & 0\% & 0\% & 0\%
\\
%$\bh_{ij}\to \bh_{ij}+\mathrm{Unif}(1,10)$
Translate each entry by $\mathrm{Unif}(1,10)$
         & 9\% & 3\% & 20\% & 26\% & 22\%
         & 0\% & 0\% & 0\% & 0\% & 0\%
         & 0\% & 0\% & 0\% & 0\% & 0\%
\\
%$\bh_{i}\to \bh_{i}/\|\bh_i\|_1$
Normalize to unit $L_1$ norm
         & 1\% & 0\% & 3\% & 6\% & 1\%
         & 22\% & \hl{56\%} & 37\% & 18\% & 20\%
         & 0\% & 0\% & 0\% & 2\% & 0\%
\\
%$\bh_{i}\to \bh_{i}/\|\bh_i\|_2$
Normalize to unit $L_2$ norm
         & 1\% & 0\% & 3\% & 6\% & 2\%
         & 22\% & \hl{57\%} & 38\% & 15\% & 19\%
         & 0\% & 0\% & 1\% & 7\% & 0\%
\\
\midrule
{\bf Average}
         & {\bf 4\%} & {\bf 1\%} & {\bf 7\%} & {\bf 10\%} & {\bf 6\%}
         & 5\% & 28\% & 12\% & 13\% & 11\%
         & 15\% & 22\% & 25\% & 26\% & 16\%
\\

\bottomrule
\end{tabular}
}
\caption{\textbf{ Sensitivity to embedding transformations (lower is better):} 
We report the fraction of independent models for which the classification changes after transforming the surrogate (averaged over 14 datasets). Lower values indicate greater robustness. PreGIP performs poorly for GIN models and under exponential and power transformations. GrOVe is inapplicable when transformations change the embedding dimension (counted as $100\%$ error). \ourmethod remains robust everywhere.}

% %We applied various transformations to each surrogate's output embeddings to create transformed versions.
% We report the fraction of independent models against which a surrogate is correctly classified, but its transformed version is not, or vice versa (averaged over 14 datasets).
% % The fractions are averaged over all 14 datasets.
% % Lower fractions imply greater robustness to transformations. 
% %Red highlights mark changes of $50\%$ or more.
% PreGIP underperforms for GIN models and under exponential and power transformations.
% GrOVe is inapplicable for transformations that change the embedding dimension (counted as $100\%$ error rate in the average).
% \ourmethod performs well under all transformations.
% % The numbers are averaged over all datasets.report this number, normalized by the 
% % For each transform, we calculate the fraction of independent models that are correctly classified \txtred{a method} correctly identifies the surrogate model in head-to-head hypothesis tests against an independent model, averaged over all datasets and GNN architectures (higher is better).

\label{tab:advTransforms}
\vspace{-1em}
\end{table*}

\begin{figure}[t]
    \centering
    \begin{subfigure}{0.3\textwidth}
        \centering
        \includegraphics[width=\textwidth]{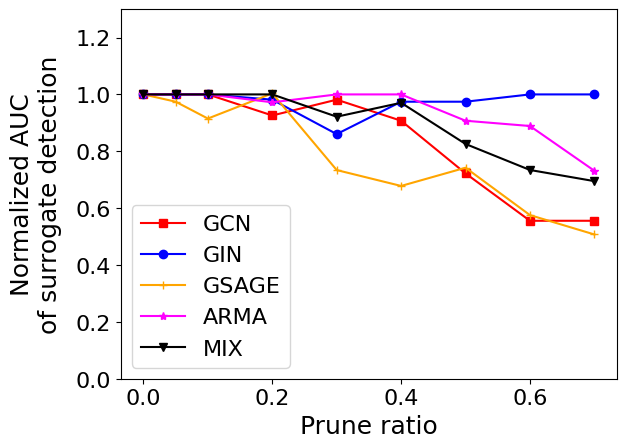}
        \caption{Effect of pruning}
        \label{fig:allTheRest:prune}
    \end{subfigure}
    \begin{subfigure}{0.3\textwidth}
        \centering
        \includegraphics[width=\textwidth]{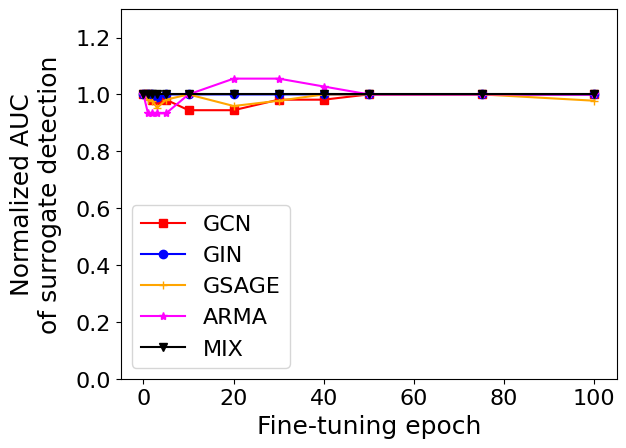}
        \caption{Effect of fine-tuning}
        \label{fig:allTheRest:finetune}
    \end{subfigure}
    \caption{{\em AUC of surrogate detection against pruning and fine-tuning attacks:}
    Results are normalized against no-tuning/fine-tuning.
    }
    %\ourmethod is robust under these attacks.}
    \label{fig:allTheRest}
    \vspace{-1em}
\end{figure}

\subsection{Accuracy under Other Adversarial Attacks}
In a {\em pruning attack}, the adversary zeros out a fraction of the surrogate model’s weights. To evaluate robustness, we compute \ourmethod’s AUC for detecting the pruned surrogate among independent models, normalized by the AUC of the unpruned model. Figure~\ref{fig:allTheRest:prune} reports the trimmed mean of these normalized AUCs across multiple datasets (Citeseer, Pubmed, DBLP, Amazon, Computers).

\ourmethod remains robust up to $\mathbf{40\%}$ pruning, after which the detection AUC declines gradually. Notably, for GIN, the AUC remains stable even with up to 70\% of weights pruned.

Next, we consider a {\em fine-tuning attack}, where the adversary updates the surrogate using a small labeled dataset. Figure~\ref{fig:allTheRest:finetune} shows that \ourmethod’s detection AUC remains stable throughout fine-tuning. Meanwhile, the downstream task accuracy saturates after 20--40 iterations (Figure~\ref{fig:finetune_dsacc} in the Appendix), indicating limited benefit from further updates. Overall, \ourmethod is robust to fine-tuning attacks.

\subsection{Additional analyses}
\begin{figure}[t]
    \centering
    \begin{subfigure}{0.3\textwidth}
        \centering
        \includegraphics[width=\textwidth]{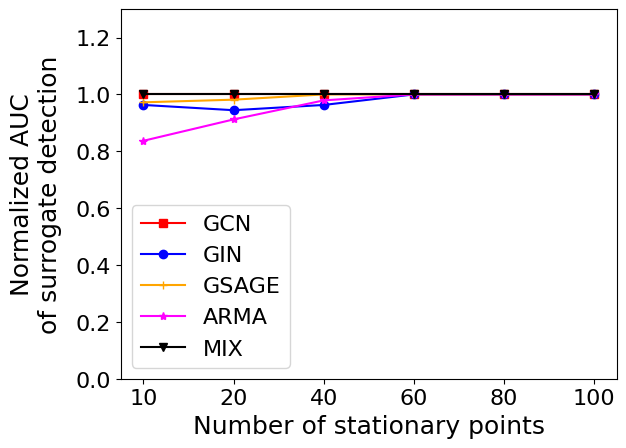}
        \caption{Varying \# of stationary points}
        \label{fig:numstatpts}
    \end{subfigure}
    \begin{subfigure}{0.3\textwidth}
        \centering
        \includegraphics[width=\textwidth]{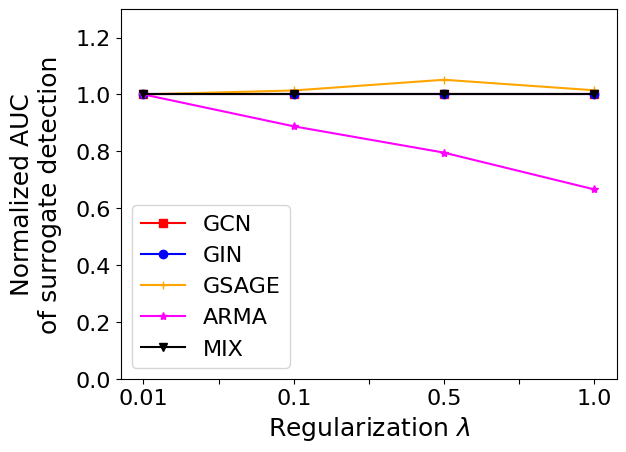}
        \caption{Effect of varying $\lambda$}
        \label{fig:varyLambda}
    \end{subfigure}
    \begin{subfigure}{0.25\textwidth}
        \centering
        \includegraphics[width=\textwidth, height=8.5em]{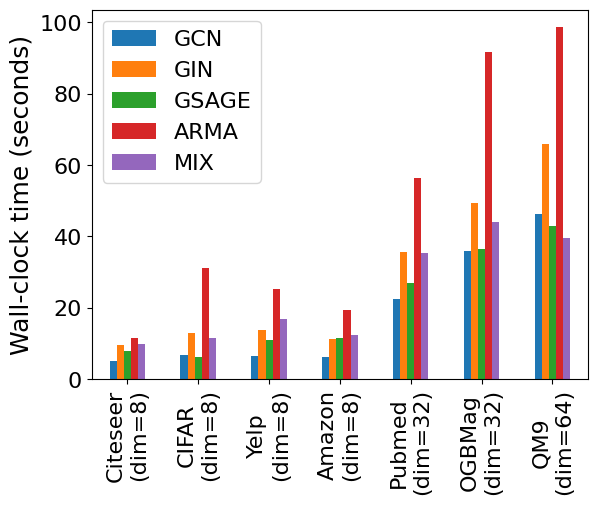}
        \caption{Time to find stat. point}
        \label{fig:time}
    \end{subfigure}
    \caption{\ourmethod needs only 20--40 stationary points to reach near-maximum AUC. AUC is largely insensitive to $\lambda$, except for ARMA where $\lambda=0.01$ is preferred.}
    \label{fig:additional}
    \vspace{-1em}
\end{figure}

\ourmethod needs only $20$--$40$ stationary points to reach near-maximum AUC across all GNN
architectures (Figure~\ref{fig:numstatpts}), consistent with Theorem~\ref{thm:special}.
% 's $e^{-2|\mathcal{T}|\gamma_M^2}$ failure probability decay. 
Surrogate detection is largely insensitive to
the regularization $\lambda$ in Equation~\ref{eq:sample}, except for ARMA where $\lambda=0.01$ is
preferred (Figure~\ref{fig:varyLambda}).
The wall-clock time per stationary point grows with embedding
dimension and is shortest for GCN, longest for ARMA (Figure~\ref{fig:time}).
% \begin{figure}
% \centering
%     \includegraphics[width=0.4\textwidth]{FIG/Time.png}
%     \caption{\textbf{ Wall-clock time for finding one stationary point.}}
%     \label{fig:time}
% \end{figure}
\section{Conclusions}
\label{sec:conc}

We introduced \ourmethod, a fingerprinting method for verifying the ownership of GNNs that output node embeddings. Our key insight is that stationary points of the embedding function are invariant under a broad class of transformations, making them a robust basis for fingerprinting. This allows \ourmethod to detect surrogate models even when the adversary changes architectures, embedding dimensions, or applies complex transformations to the embeddings.

We complement this idea with a practical sampling and testing procedure, along with theoretical guarantees for surrogate detection. Empirically, across 14 datasets and 5 GNN architectures, \ourmethod consistently achieves high detection accuracy and outperforms existing watermarking and fingerprinting methods. 
%As shown in Tables 1–2 and Figures 2–3, it remains robust under model extraction, embedding transformations, pruning, and fine-tuning attacks :contentReference[oaicite:0]{index=0}.

One direction of future work is to relax or validate assumptions about the transformation function and the distribution of stationary points.
Also, improving the efficiency of stationary point discovery, especially for very large graphs or higher-dimensional embeddings, would further enhance scalability.

\bibliographystyle{plainnat}
\bibliography{bibliography}

@inproceedings{pregip,
	address = {Toronto ON Canada},
	title = {{PreGIP}: {Watermarking} the {Pretraining} of {Graph} {Neural} {Networks} for {Deep} {IP} {Protection}},
	isbn = {979-8-4007-1454-2},
	shorttitle = {{PreGIP}},
	url = {https://dl.acm.org/doi/10.1145/3711896.3737089},
	doi = {10.1145/3711896.3737089},
	language = {en},
	urldate = {2025-10-07},
	booktitle = {Proceedings of the 31st {ACM} {SIGKDD} {Conference} on {Knowledge} {Discovery} and {Data} {Mining} {V}.2},
	publisher = {ACM},
	author = {Dai, Enyan and Lin, Minhua and Wang, Suhang},
	month = aug,
	year = {2025},
	pages = {415--426},
}

@inproceedings{adi_turning_2018,
	title = {Turning {Your} {Weakness} {Into} a {Strength}: {Watermarking} {Deep} {Neural} {Networks} by {Backdooring}},
	language = {en},
	booktitle = {{SEC}'18: {Proceedings} of the 27th {USENIX} {Conference} on {Security} {Symposium}},
	author = {Adi, Yossi and Baum, Carsten and Pinkas, Benny and Keshet, Joseph},
	year = {2018},
}

@inproceedings{xu_watermarking_2023,
	title = {Watermarking {Graph} {Neural} {Networks} based on {Backdoor} {Attacks}},
	url = {https://ieeexplore.ieee.org/document/10190545/},
	doi = {10.1109/EuroSP57164.2023.00072},
	urldate = {2025-10-13},
	booktitle = {2023 {IEEE} 8th {European} {Symposium} on {Security} and {Privacy} ({EuroS}\&{P})},
	author = {Xu, Jing and Koffas, Stefanos and Ersoy, Oğuzhan and Picek, Stjepan},
	month = jul,
	year = {2023},
	pages = {1179--1197},
}

@inproceedings{uchida_embedding_2017,
	address = {New York, NY, USA},
	series = {{ICMR} '17},
	title = {Embedding {Watermarks} into {Deep} {Neural} {Networks}},
	isbn = {978-1-4503-4701-3},
	url = {https://dl.acm.org/doi/10.1145/3078971.3078974},
	doi = {10.1145/3078971.3078974},
	urldate = {2025-10-13},
	booktitle = {Proceedings of the 2017 {ACM} on {International} {Conference} on {Multimedia} {Retrieval}},
	publisher = {Association for Computing Machinery},
	author = {Uchida, Yusuke and Nagai, Yuki and Sakazawa, Shigeyuki and Satoh, Shin'ichi},
	month = jun,
	year = {2017},
	pages = {269--277},
}

@inproceedings{chen_deepmarks_2019,
	address = {New York, NY, USA},
	series = {{ICMR} '19},
	title = {{DeepMarks}: {A} {Secure} {Fingerprinting} {Framework} for {Digital} {Rights} {Management} of {Deep} {Learning} {Models}},
	isbn = {978-1-4503-6765-3},
	shorttitle = {{DeepMarks}},
	url = {https://dl.acm.org/doi/10.1145/3323873.3325042},
	doi = {10.1145/3323873.3325042},
	urldate = {2025-10-13},
	booktitle = {Proceedings of the 2019 on {International} {Conference} on {Multimedia} {Retrieval}},
	publisher = {Association for Computing Machinery},
	author = {Chen, Huili and Rouhani, Bita Darvish and Fu, Cheng and Zhao, Jishen and Koushanfar, Farinaz},
	month = jun,
	year = {2019},
	pages = {105--113},
}

@inproceedings{darvish_rouhani_deepsigns_2019,
	address = {New York, NY, USA},
	series = {{ASPLOS} '19},
	title = {{DeepSigns}: {An} {End}-to-{End} {Watermarking} {Framework} for {Ownership} {Protection} of {Deep} {Neural} {Networks}},
	isbn = {978-1-4503-6240-5},
	shorttitle = {{DeepSigns}},
	url = {https://dl.acm.org/doi/10.1145/3297858.3304051},
	doi = {10.1145/3297858.3304051},
	urldate = {2025-10-13},
	booktitle = {Proceedings of the {Twenty}-{Fourth} {International} {Conference} on {Architectural} {Support} for {Programming} {Languages} and {Operating} {Systems}},
	publisher = {Association for Computing Machinery},
	author = {Darvish Rouhani, Bita and Chen, Huili and Koushanfar, Farinaz},
	month = apr,
	year = {2019},
	pages = {485--497},
}

@inproceedings{zhao_watermarking_2021,
	title = {Watermarking {Graph} {Neural} {Networks} by {Random} {Graphs}},
	url = {https://ieeexplore.ieee.org/document/9486352},
	doi = {10.1109/ISDFS52919.2021.9486352},
	urldate = {2025-10-13},
	booktitle = {2021 9th {International} {Symposium} on {Digital} {Forensics} and {Security} ({ISDFS})},
	author = {Zhao, Xiangyu and Wu, Hanzhou and Zhang, Xinpeng},
	month = jun,
	year = {2021},
	pages = {1--6},
}

@article{kwon_blindnet_2022,
	title = {{BlindNet} backdoor: {Attack} on deep neural network using blind watermark},
	volume = {81},
	issn = {1380-7501},
	shorttitle = {{BlindNet} backdoor},
	url = {https://doi.org/10.1007/s11042-021-11135-0},
	doi = {10.1007/s11042-021-11135-0},
	number = {5},
	urldate = {2025-10-14},
	journal = {Multimedia Tools Appl.},
	author = {Kwon, Hyun and Kim, Yongchul},
	month = feb,
	year = {2022},
	pages = {6217--6234},
}

@misc{zhao_recipe_2023,
	title = {A {Recipe} for {Watermarking} {Diffusion} {Models}},
	url = {http://arxiv.org/abs/2303.10137},
	doi = {10.48550/arXiv.2303.10137},
	urldate = {2025-10-14},
	publisher = {arXiv},
	author = {Zhao, Yunqing and Pang, Tianyu and Du, Chao and Yang, Xiao and Cheung, Ngai-Man and Lin, Min},
	month = oct,
	year = {2023},
	note = {arXiv:2303.10137 [cs]},
}

@article{dai_comprehensive_2024,
	title = {A {Comprehensive} {Survey} on {Trustworthy} {Graph} {Neural} {Networks}: {Privacy}, {Robustness}, {Fairness}, and {Explainability}},
	volume = {21},
	issn = {2731-538X, 2731-5398},
	shorttitle = {A {Comprehensive} {Survey} on {Trustworthy} {Graph} {Neural} {Networks}},
	url = {https://link.springer.com/10.1007/s11633-024-1510-8},
	doi = {10.1007/s11633-024-1510-8},
	language = {en},
	number = {6},
	urldate = {2025-10-14},
	journal = {Machine Intelligence Research},
	author = {Dai, Enyan and Zhao, Tianxiang and Zhu, Huaisheng and Xu, Junjie and Guo, Zhimeng and Liu, Hui and Tang, Jiliang and Wang, Suhang},
	month = dec,
	year = {2024},
	pages = {1011--1061},
}

@misc{grove,
	title = {{GrOVe}: {Ownership} {Verification} of {Graph} {Neural} {Networks} using {Embeddings}},
	shorttitle = {{GrOVe}},
	url = {http://arxiv.org/abs/2304.08566},
	doi = {10.48550/arXiv.2304.08566},
	urldate = {2025-10-14},
	publisher = {arXiv},
	author = {Waheed, Asim and Duddu, Vasisht and Asokan, N.},
	month = sep,
	year = {2023},
	note = {arXiv:2304.08566 [cs]},
}

@inproceedings{shen_model_2022,
	title = {Model {Stealing} {Attacks} {Against} {Inductive} {Graph} {Neural} {Networks}},
	isbn = {978-1-6654-1316-9},
	url = {https://www.computer.org/csdl/proceedings-article/sp/2022/131600b031/1FlQwoy12IU},
	doi = {10.1109/SP46214.2022.9833607},
	language = {English},
	urldate = {2025-10-14},
	publisher = {IEEE Computer Society},
	author = {Shen, Yun and He, Xinlei and Han, Yufei and Zhang, Yang},
	month = may,
	year = {2022},
	pages = {1175--1192},
}

@misc{lukas_sok_2021,
	title = {{SoK}: {How} {Robust} is {Image} {Classification} {Deep} {Neural} {Network} {Watermarking}? ({Extended} {Version})},
	shorttitle = {{SoK}},
	url = {http://arxiv.org/abs/2108.04974},
	doi = {10.48550/arXiv.2108.04974},
	urldate = {2025-10-15},
	publisher = {arXiv},
	author = {Lukas, Nils and Jiang, Edward and Li, Xinda and Kerschbaum, Florian},
	month = aug,
	year = {2021},
	note = {arXiv:2108.04974 [cs]},
}

@misc{defazio_adversarial_2019,
	title = {Adversarial {Model} {Extraction} on {Graph} {Neural} {Networks}},
	url = {http://arxiv.org/abs/1912.07721},
	doi = {10.48550/arXiv.1912.07721},
	urldate = {2026-01-24},
	publisher = {arXiv},
	author = {DeFazio, David and Ramesh, Arti},
	month = dec,
	year = {2019},
	note = {arXiv:1912.07721 [cs]},
}

@inproceedings{lee_defending_2019,
	title = {Defending {Against} {Neural} {Network} {Model} {Stealing} {Attacks} {Using} {Deceptive} {Perturbations}},
	url = {https://ieeexplore.ieee.org/document/8844598},
	doi = {10.1109/SPW.2019.00020},
	urldate = {2026-01-24},
	booktitle = {2019 {IEEE} {Security} and {Privacy} {Workshops} ({SPW})},
	author = {Lee, Taesung and Edwards, Benjamin and Molloy, Ian and Su, Dong},
	month = may,
	year = {2019},
	pages = {43--49},
}

@inproceedings{juuti_prada_2019,
	title = {{PRADA}: {Protecting} {Against} {DNN} {Model} {Stealing} {Attacks}},
	shorttitle = {{PRADA}},
	url = {https://ieeexplore.ieee.org/document/8806737},
	doi = {10.1109/EuroSP.2019.00044},
	urldate = {2026-01-24},
	booktitle = {2019 {IEEE} {European} {Symposium} on {Security} and {Privacy} ({EuroS}\&{P})},
	author = {Juuti, Mika and Szyller, Sebastian and Marchal, Samuel and Asokan, N.},
	month = jun,
	year = {2019},
	pages = {512--527},
}

@inproceedings{lukas_deep_2021,
	title = {Deep {Neural} {Network} {Fingerprinting} by {Conferrable} {Adversarial} {Examples}},
	language = {en},
	urldate = {2026-01-24},
	author = {Lukas, Nils and Zhang, Yuxuan and Kerschbaum, Florian},
	year = {2021},
	note = {arXiv:1912.00888 [cs]},
}

@inproceedings{peng_fingerprinting_2022,
	address = {New Orleans, LA, USA},
	title = {Fingerprinting {Deep} {Neural} {Networks} {Globally} via {Universal} {Adversarial} {Perturbations}},
	copyright = {https://doi.org/10.15223/policy-029},
	isbn = {978-1-6654-6946-3},
	url = {https://ieeexplore.ieee.org/document/9879052/},
	doi = {10.1109/CVPR52688.2022.01307},
	language = {en},
	urldate = {2026-01-24},
	booktitle = {2022 {IEEE}/{CVF} {Conference} on {Computer} {Vision} and {Pattern} {Recognition} ({CVPR})},
	publisher = {IEEE},
	author = {Peng, Zirui and Li, Shaofeng and Chen, Guoxing and Zhang, Cheng and Zhu, Haojin and Xue, Minhui},
	month = jun,
	year = {2022},
	pages = {13420--13429},
}

@inproceedings{papernot_practical_2017,
	address = {New York, NY, USA},
	series = {{ASIA} {CCS} '17},
	title = {Practical {Black}-{Box} {Attacks} against {Machine} {Learning}},
	isbn = {978-1-4503-4944-4},
	url = {https://dl.acm.org/doi/10.1145/3052973.3053009},
	doi = {10.1145/3052973.3053009},
	urldate = {2026-01-24},
	booktitle = {Proceedings of the 2017 {ACM} on {Asia} {Conference} on {Computer} and {Communications} {Security}},
	publisher = {Association for Computing Machinery},
	author = {Papernot, Nicolas and McDaniel, Patrick and Goodfellow, Ian and Jha, Somesh and Celik, Z. Berkay and Swami, Ananthram},
	month = apr,
	year = {2017},
	pages = {506--519},
}

@article{krishna_thieves_2020,
	title = {Thieves on {Sesame} {Street}! {Model} {Extraction} of {BERT}-based {APIs}},
	language = {en},
	author = {Krishna, Kalpesh and Tomar, Gaurav Singh and Parikh, Ankur P and Papernot, Nicolas and Iyyer, Mohit},
	year = {2020},
}

@inproceedings{maini_dataset_2021,
	title = {Dataset {Inference}: {Ownership} {Resolution} in {Machine} {Learning}},
	language = {en},
	author = {Maini, Pratyush and Yaghini, Mohammad and Papernot, Nicolas},
	year = {2021},
}

@inproceedings{peng_are_2023,
	address = {Toronto, Canada},
	title = {Are {You} {Copying} {My} {Model}? {Protecting} the {Copyright} of {Large} {Language} {Models} for {EaaS} via {Backdoor} {Watermark}},
	shorttitle = {Are {You} {Copying} {My} {Model}?},
	url = {https://aclanthology.org/2023.acl-long.423/},
	doi = {10.18653/v1/2023.acl-long.423},
	urldate = {2026-01-24},
	booktitle = {Proceedings of the 61st {Annual} {Meeting} of the {Association} for {Computational} {Linguistics} ({Volume} 1: {Long} {Papers})},
	publisher = {Association for Computational Linguistics},
	author = {Peng, Wenjun and Yi, Jingwei and Wu, Fangzhao and Wu, Shangxi and Bin Zhu, Bin and Lyu, Lingjuan and Jiao, Binxing and Xu, Tong and Sun, Guangzhong and Xie, Xing},
	editor = {Rogers, Anna and Boyd-Graber, Jordan and Okazaki, Naoaki},
	month = jul,
	year = {2023},
	pages = {7653--7668},
}

@inproceedings{zhang_gnnguard_2020,
	address = {Red Hook, NY, USA},
	series = {{NIPS} '20},
	title = {{GNNGUARD}: defending graph neural networks against adversarial attacks},
	isbn = {978-1-7138-2954-6},
	shorttitle = {{GNNGUARD}},
	url = {https://dl.acm.org/doi/10.5555/3495724.3496501},
	urldate = {2026-01-24},
	booktitle = {Proceedings of the 34th {International} {Conference} on {Neural} {Information} {Processing} {Systems}},
	publisher = {Curran Associates Inc.},
	author = {Zhang, Xiang and Zitnik, Marinka},
	month = dec,
	year = {2020},
	pages = {9263--9275},
}

@inproceedings{
deng2022garnet,
title={{GARNET}: Reduced-Rank Topology Learning for Robust and Scalable Graph Neural Networks},
author={Chenhui Deng and Xiuyu Li and Zhuo Feng and Zhiru Zhang},
booktitle={Learning on Graphs Conference},
year={2022},
url={https://openreview.net/forum?id=kvwWjYQtmw}
}

@inproceedings{zhang_inference_2022,
	title = {Inference {Attacks} {Against} {Graph} {Neural} {Networks}},
	language = {en},
	booktitle = {31st {USENIX} {Security} {Symposium}},
	author = {Zhang, Zhikun and Chen, Min and Backes, Michael and Shen, Yun and Zhang, Yang},
	year = {2022},
}

@inproceedings{wang_group_2022,
	address = {New York, NY, USA},
	series = {{CCS} '22},
	title = {Group {Property} {Inference} {Attacks} {Against} {Graph} {Neural} {Networks}},
	isbn = {978-1-4503-9450-5},
	url = {https://dl.acm.org/doi/10.1145/3548606.3560662},
	doi = {10.1145/3548606.3560662},
	urldate = {2026-01-24},
	booktitle = {Proceedings of the 2022 {ACM} {SIGSAC} {Conference} on {Computer} and {Communications} {Security}},
	publisher = {Association for Computing Machinery},
	author = {Wang, Xiuling and Wang, Wendy Hui},
	month = nov,
	year = {2022},
	pages = {2871--2884},
}

@inproceedings{kipf2017semi,
  title={Semi-Supervised Classification with Graph Convolutional Networks},
  author={Kipf, Thomas N. and Welling, Max},
  booktitle={International Conference on Learning Representations (ICLR)},
  year={2017}
}

@inproceedings{xu2019how,
  title={How Powerful are  Graph Neural Networks?},
  author={Xu, Keyulu and Hu, Weihua and Leskovec, Jure and Jegelka, Stefanie},
  booktitle={International Conference on Learning Representations (ICLR)},
  year={2019}
}

@inproceedings{hamilton2017inductive,
     author = {Hamilton, William L. and Ying, Rex and Leskovec, Jure},
     title = {Inductive Representation Learning on Large Graphs},
     booktitle = {NIPS},
     year = {2017}
}

@ARTICLE{bianchi2022conv,
  author={Bianchi, Filippo Maria and Grattarola, Daniele and Livi, Lorenzo and Alippi, Cesare},
  journal={IEEE Transactions on Pattern Analysis and Machine Intelligence}, 
  title={Graph Neural Networks With Convolutional ARMA Filters}, 
  year={2022},
  volume={44},
  number={7},
  pages={3496-3507},
  doi={10.1109/TPAMI.2021.3054830}}

@inproceedings{abu-el-haija_mixhop_2019,
	title = {{MixHop}: {Higher}-{Order} {Graph} {Convolutional} {Architectures}  via {Sparsified} {Neighborhood} {Mixing}},
	language = {en},
	author = {Abu-El-Haija, Sami and Perozzi, Bryan and Kapoor, Amol and Alipourfard, Nazanin and Lerman, Kristina and Harutyunyan, Hrayr and Steeg, Greg Ver and Galstyan, Aram},
	year = {2019},
}

@article{bennet2021nevergrad,
  author = {Bennet, Pauline and Doerr, Carola and Moreau, Antoine and Rapin, Jeremy and Teytaud, Fabien and Teytaud, Olivier},
  title = {Nevergrad: Black-Box Optimization Platform},
  journal = {ACM SIGEVOlution},
  volume = {14},
  number = {1},
  pages = {8--15},
  year = {2021},
  issn = {1931-8499},
  doi = {10.1145/3460310.3460312},
  publisher = {ACM},
  address = {New York, NY, USA},
  url = {https://doi.org}
}

@inproceedings{wu2024securing,
  title={Securing graph neural networks in MLaaS: A comprehensive realization of query-based integrity verification},
  author={Wu, Bang and Yuan, Xingliang and Wang, Shuo and Li, Qi and Xue, Minhui and Pan, Shirui},
  booktitle={2024 IEEE Symposium on Security and Privacy (SP)},
  pages={2534--2552},
  year={2024},
  organization={IEEE}
}

@inproceedings{8953972,
  author={He, Zecheng and Zhang, Tianwei and Lee, Ruby},
  booktitle={2019 IEEE/CVF Conference on Computer Vision and Pattern Recognition (CVPR)}, 
  title={Sensitive-Sample Fingerprinting of Deep Neural Networks}, 
  year={2019},
  volume={},
  number={},
  pages={4724-4732},
  doi={10.1109/CVPR.2019.00486}}

@article{kuttichira2022verification,
  title={Verification of integrity of deployed deep learning models using Bayesian optimization},
  author={Kuttichira, Deepthi Praveenlal and Gupta, Sunil and Nguyen, Dang and Rana, Santu and Venkatesh, Svetha},
  journal={Knowledge-based systems},
  volume={241},
  pages={108238},
  year={2022},
  publisher={Elsevier}
}

@inproceedings{wang2023publiccheck,
  title={Publiccheck: Public integrity verification for services of run-time deep models},
  author={Wang, Shuo and Abuadbba, Sharif and Agarwal, Sidharth and Moore, Kristen and Sun, Ruoxi and Xue, Minhui and Nepal, Surya and Camtepe, Seyit and Kanhere, Salil},
  booktitle={2023 IEEE Symposium on Security and Privacy (SP)},
  pages={1348--1365},
  year={2023},
  organization={IEEE}
}

%%%%%%%%%%%%%%%%%%%%%%%%%%%%%%%%%%%%%%%%%%%%%%%%%%%%%%%%%%%%%%%%%%%%%%%%%%%%%%%
%%%%%%%%%%%%%%%%%%%%%%%%%%%%%%%%%%%%%%%%%%%%%%%%%%%%%%%%%%%%%%%%%%%%%%%%%%%%%%%
% APPENDIX
%%%%%%%%%%%%%%%%%%%%%%%%%%%%%%%%%%%%%%%%%%%%%%%%%%%%%%%%%%%%%%%%%%%%%%%%%%%%%%%
%%%%%%%%%%%%%%%%%%%%%%%%%%%%%%%%%%%%%%%%%%%%%%%%%%%%%%%%%%%%%%%%%%%%%%%%%%%%%%%
\appendix
\section{Remarks and Proofs}
\label{sec:app:proofs}
% \begin{remark}
% The sup-norm requirement in Equation~\ref{eq:epsilon} can be relaxed to the supremum over a subset of graphs $(G,X)$ with a high probability under the data distribution.
% We keep the form of Equation~\ref{eq:epsilon} to simplify the exposition.
% \label{rem:supnorm}
% \end{remark}

\begin{remark}
    In Algorithm~\ref{alg:general}, we can use an alternative formula for $\hat{\beta}_Z$:
    \begin{align}
        \hat{\beta}_Z &= \text{Mean}\left(\{ \text{percentile of $q_Z(t)$ in $U$}\}_{t\in T}\right)\nonumber\\
        \text{where }U &= \{q_Z(t)\mid t\in R\}.\label{eq:percent}
    \end{align}
    % \txtblue{RN: Will it be confusing to use $Q$ here, but $\mathcal{Q}$ for the query tuple set? DC: Changed to U?}
    
    Here, $U$ contains samples of $q_Z(t)$ for $t\sim\mathcal{D}_{exp}$, supported on the positive real line.
    Now, consider a point $t\in T$.
    If $t$ is a stationary point for $Z$, $q_Z(t)\approx 0$, so the percentile of $q_Z(t)$ in $U$ is also small.
   Otherwise, $t$ is an idiosyncratic stationary point of $M$.
   In this situation, we expect $q_Z(t)$ to be distributed like a tuple drawn from $\mathcal{D}_{exp}$, i.e., with a percentile score distributed uniformly between $0$ and $100$.
    Thus, $\hat{\beta}_Z$ is small if $Z$ is a surrogate but high if $Z$ is independent.
   This mirrors the behavior of $\beta_Z$ (Equation~\ref{eq:betaz}) while avoiding division-related instability.   
   We use this variant in our experiments.
   \label{rem:actual_betaZ}
\end{remark}

\begin{lemma}[Closure under composition]\label{lem:closure}
If $f: \mathbb{R}^d \to \mathbb{R}^{d'}$ and $g: \mathbb{R}^{d''} \to \mathbb{R}^d$ each satisfy Assumption~\ref{A:invertible}, then $f \circ g$ also satisfies Assumption~\ref{A:invertible}.
\end{lemma}
\begin{proof}[Proof of Lemma~\ref{lem:closure}]
Let $w$ be any unit vector and $h$ any point in the domain of $g$.
By the chain rule,
\[
  \nabla_w (f \circ g)\big|_h
  = J_f\big|_{g(h)} \cdot J_g\big|_h \, w
  = J_f\big|_{g(h)} \cdot \nabla_w g\big|_h.
\]
Since $g$ satisfies Assumption~\ref{A:invertible}, $\nabla_w g\big|_h \neq 0$.
Let $v := \nabla_w g\big|_h$; then $v \neq 0$, so $v / \|v\|$ is a unit vector.
Since $f$ satisfies Assumption~\ref{A:invertible},
\[
  J_f\big|_{g(h)} \, v
  = \|v\| \cdot \nabla_{v/\|v\|} f\big|_{g(h)}
  \neq 0.
\]
Hence $\nabla_w (f \circ g)\big|_h \neq 0$, so $f \circ g$ satisfies Assumption~\ref{A:invertible}.
\end{proof}
\begin{proof}[Proof of Lemma~\ref{lem:stat_invariant}]
Let $J^\prime$, $J^\phi$, and $J$ be the Jacobians of the functions $\bh_i^\prime(X)$, $\phi(\bh)$, and $\bh(X)$ respectively (these exist by Assumptions~\ref{A:bhi} and~\ref{A:invertible}). 
By the chain rule,
\begin{align*}
\nabla_{\bw}\bh_i^\prime(X) &= \left.J^\prime\right|_{X}\bw \\
    &= \left.J^\phi\right|_{\bh_i(X)} \left.J\right|_{X} \bw = \left.J^\phi\right|_{\bh_i(X)} \nabla_{\bw}\bh_i(X),
\end{align*}
Hence, $\nabla_{\bw}\bh_i(X)=0 \Rightarrow \nabla_{\bw}\bh_i^\prime(X)=0$.
Conversely, suppose $\nabla_{\bw}\bh_i^\prime(X)=0$.
Then, $\left.J^\phi\right|_{\bh_i(X)} \bv = \bm{0}$, where $\bv:=\nabla_{\bw}\bh_i(X)$.
If $\bv\neq \bm{0}$, then $\bm{0}=\left.J^\phi\right|_{\bh_i(X)} \bv = \|\bv\|\cdot \left.\nabla_{\bv/\|\bv\|}\phi\right|_{\bh_i(X)}$, which is impossible by Assumption~\ref{A:invertible}.
Hence, $\nabla_{\bw}\bh_i(X)=\bv=\bm{0}$.
\end{proof}

\begin{proof}[Proof of Theorem~\ref{thm:special}]
By Lemma~\ref{lem:stat_invariant}, Algorithm~\ref{alg:special} fails iff $Z$ is independent but $T\subseteq \mathcal{S}(Z)$.
Suppose $Z$ is an independent model.
Each $t\in T$ is drawn independently using measure $\mu_M$.
We have, $E_{t\sim\mu_M}[\mathds{1}_{t\in\mathcal{S}(Z)}]=P(t\in \mathcal{S}(Z))\leq 1-\gamma_M$ by Assumption~\ref{A:idiosyncratic}.
Algorithm~\ref{alg:special} fails with probability
\begin{align*}
    P\left(\frac{|T\cap\mathcal{S}(Z)|}{|T|}=1\right) &\leq (1-\gamma_M)^{|T|} \leq \exp\left(-\gamma_M\cdot T\right).
    % &\quad\leq P\left(\frac{|T\cap\mathcal{S}(Z)|}{|T|}\geq E_{t\sim\mu_M}[\mathds{1}_{t\in\mathcal{S}(Z)}] + \gamma_M\right)\\
    % &\quad\leq \exp\left(-2\gamma_M^2 |T|\right),
\end{align*}
%by Hoeffding's bound.
\end{proof}

\begin{proof}[Proof of Lemma~\ref{lem:betam}]
%By Equation~\ref{eq:q}, 
For any feasible tuple $t=(G,X,i,\bw)$, 
\begin{align*}
    q_M(t) &=\frac{\|\bh_i^{(\delta\bw)}(X) - \bh_i(X)\|}{\|\bh_i(X)\|}
    =\frac{\delta\cdot \|\nabla_{\bw}\bh_i\| + O(\delta^2)}{\|\bh_i\|}.
\end{align*}
For $t\in\mathcal{S}(M)$, $\|\nabla_{\bw}\bh_i\|=0$ so $q_M(t)=O(\delta^2)$ (since $\|\bh_i\|\geq \ualpha_M$ by Assumption~\ref{A:bhi}).
Now, $u_M:=E_{t\sim\mathcal{D}_{exp}}\|\nabla_{\bw}\bh_i\|>0$ because the GNN's outputs must be affected by changing inputs for at least some graphs.
So, $E_{t\sim\mathcal{D}_{exp}}\geq \frac{\delta u_M + O(\delta^2)}{\oalpha_M}$.
Hence,
\begin{align*}
    \beta_M\leq \frac{O(\delta^2)}{\delta u_M/\oalpha_M + O(\delta^2)} = O(\delta).
\end{align*}
\end{proof}

\begin{theorem}
For any model $F$, $\beta_F$ is unchanged under any combination of 
\begin{align*}
    \text{rotation: } & \bh\to O\bh & (O \text{ orthonormal})\\
    \text{scaling: } & \bh\to c\bh & (c\neq 0)\\
    \text{up-projection: } & \bh\to [\bh^T, \underbrace{0, \ldots, 0}_{\ell \text{times}}]^T & (\text{arbitrary } \ell).
\end{align*}
%and also for any composition of these transformations.
\label{thm:betaZ}
\vspace{-1em}
\end{theorem}
\begin{proof}[Proof of Theorem~\ref{thm:betaZ}]
    $q_F(t)$ is clearly invariant under these transformations, so $\beta_F$ is too.
\end{proof}

\begin{proof}[Proof of Lemma~\ref{lem:betamprime}]
    \begin{align*}
        \|&\phi(\hbh_i^{(\delta\bw)}) - \phi(\hbh_i)\| \\
        &\leq C\|\hbh_i^{(\delta\bw)} - \hbh_i\|\\
        &\leq C\left[ \|\hbh_i^{(\delta\bw)}-\bh_i^{(\delta\bw)}\| + \|\bh_i^{(\delta\bw)}-\bh_i\| + \|\hbh_i-\bh_i\|\right]\\
        &\leq C\left[ 2\epsilon + \delta\|\nabla_{\bw}\bh_i\| + O(\delta^2)\right]\\
        &\leq C\left[ \delta\|\nabla_{\bw}\bh_i\| + O(\epsilon+\delta^2)\right].
    \end{align*}
    Also,
    \begin{align*}
        \|\phi(\hbh_i)\| \geq c\|\hbh_i\| \geq c(\|\bh_i\|-\epsilon) \geq c(\ualpha_M-\epsilon) \geq c\ualpha_M/2
    \end{align*}
    for $\epsilon$ small enough, where the first inequality follows from Assumption~\ref{A:phistrong}.
    Hence, 
    \begin{align*}
        q_{M^\prime}(t) &\leq \frac{C\left[ \delta\|\nabla_{\bw}\bh_i\| + O(\epsilon+\delta^2)\right]}{c\ualpha_M/2}.\\
        \text{Similarly, }
        q_{M^\prime}(t) &\geq \frac{c\left[ \delta\|\nabla_{\bw}\bh_i\| + O(\epsilon+\delta^2)\right]}{2C\oalpha_M}.\\
    \end{align*}
    Define $u_M:=E_{t\sim\mathcal{D}_{exp}}\|\nabla_{\bw}\bh_i\|>0$ as in the proof of Lemma~\ref{lem:betam}.
    Then, 
    \begin{align*}
        E_{t\sim\mu_M}\left[q_{M^\prime}(t)\right] &\leq \frac{C}{c\ualpha/2}\cdot O(\epsilon+\delta^2) & \\
        E_{t\sim\mathcal{D}_{exp}}\left[q_{M^\prime}(t)\right] &\geq \frac{c(\delta\cdot u_M + O(\epsilon+\delta^2))}{2C\oalpha_M} &\\
        \text{So, }
        \beta_{M^\prime} &\leq \left(\frac{C}{c}\right)^2\cdot O\left( \frac{(\epsilon+\delta^2)}{\delta}\right) & (\text{since }\epsilon=o(\delta))\\
        &\leq \left(\frac{C}{c}\right)^2 O\left(\max(\epsilon/\delta, \delta)\right). &
    \end{align*}
\end{proof}

\begin{proof}[Proof of Lemma~\ref{lem:betai}]
    As in the proof of Lemma~\ref{lem:betam}, we define $u_I:=E_{t\sim\mathcal{D}_{exp}}\|\nabla_{\bw}\bh_i(G, X; I)\|>0$ and find that  
    \begin{align}
        E_{t\sim\mathcal{D}_{exp}}[q_I(t)] &\geq \frac{\delta u_I + O(\delta^2)}{\oalpha_I} = \Theta(\delta).
        \label{eq:dexpqi}
    \end{align}
    For $t\in\mathcal{S}(I)$, we have
    \begin{align}
        q_I(t) &\geq \frac{\delta\|\nabla_{\bw}\bh_i(G,X; I)\|+O(\delta^2)}{\oalpha_I} = O(\delta^2).
        \label{eq:qitinsi}
    \end{align}
    Now,
    \begin{align*}
        E_{t\sim\mu_M}\left[ q_I(t) \right]
        &= P(t\in\mathcal{S}(I))\cdot E_t\left[ q_I(t)\mid t\in\mathcal{S}(I) \right]\\
        &+ P(t\notin\mathcal{S}(I))\cdot E_t\left[ q_I(t)\mid t\notin\mathcal{S}(I) \right]\\
        &=P(\mathcal{S}(M)\cap\mathcal{S}(I))\cdot O(\delta^2)\\
        &+P(\mathcal{S}(M)\setminus\mathcal{S}(I))\cdot E_{t\sim\mathcal{D}_{exp}}\left[ q_I(t)\right]\\
        &\geq \gamma_M\cdot E_{t\sim\mathcal{D}_{exp}}\left[ q_I(t)\right] + O(\delta^2),\\
        \text{So, }\beta_I 
        &\geq \frac{\gamma_M\cdot E_{t\sim\mathcal{D}_{exp}}\left[ q_I(t)\right] + O(\delta^2)}{E_{t\sim\mathcal{D}_{exp}}\left[ q_I(t)\right]}\\
        &\geq \gamma_M + O(\delta),
    \end{align*}
    where we used Equation~\ref{eq:qitinsi} and Assumption~\ref{A:q_idio} in the second equality, and Equation~\ref{eq:dexpqi} in the last step.
\end{proof}

\begin{proof}[Proof of Theorem~\ref{thm:general}]
    Since $\bh_i$ is upper- and lower-bounded (Assumption~\ref{A:bhi}), so is $q_Z(t)$.
    %(Equation~\ref{eq:q}).
    Then, the numerator and denominator of $\hat{\beta}_Z$ converge to their counterparts for $\beta_Z$ by Hoeffding bounds, and the result follows.
\end{proof}

\begin{table*}[ht]
\centering
\small
\begin{tabular}{l|ccccc|ccccc}
\toprule
& \multicolumn{5}{c|}{\textbf{\ourmethod}} & \multicolumn{5}{c}{\textbf{PreGIP}}\\
\textbf{Dataset} 
    & \textbf{GCN} & \textbf{GIN} & \textbf{GSage} & \textbf{ARMA} & \textbf{MixHop}
    & \textbf{GCN} & \textbf{GIN} & \textbf{GSage} & \textbf{ARMA} & \textbf{MixHop} \\
\midrule
Citeseer
         & 1.00 & 1.00 & 1.00 & 1.00 & 1.00
         & 0.60 & \hl{0.40} & 1.00 & 0.80 & 0.95
\\
OGBMag
         & 1.00 & 1.00 & 1.00 & 1.00 & 1.00
         & 1.00 & \hl{0.00} & 0.60 & 0.95 & 0.85
\\
HIV
         & 1.00 & 1.00 & 0.64 & 0.78 & 1.00
         & 0.50 & 0.70 & 0.90 & 0.60 & 0.80
\\
Yelp
         & 1.00 & 1.00 & 1.00 & 0.94 & 1.00
         & 0.90 & \hl{0.00} & 0.55 & 1.00 & 0.90
\\
MNIST
         & 1.00 & 1.00 & 1.00 & 0.78 & 1.00
         & 1.00 & \hl{0.00} & 0.60 & 0.85 & \hl{0.15}
\\
BBBP
         & 1.00 & 1.00 & 0.93 & 1.00 & 1.00
         & 0.95 & \hl{0.10} & 0.70 & \hl{0.00} & 1.00
\\
Pubmed
         & 1.00 & 1.00 & 0.79 & 0.83 & 0.93
         & 0.80 & \hl{0.25} & \hl{0.30} & 0.70 & 0.90
\\
QM9
         & 1.00 & 1.00 & 1.00 & 0.78 & 1.00
         & 0.85 & \hl{0.00} & 1.00 & \hl{0.00} & 0.60
\\
Fin
         & 1.00 & 0.86 & 0.86 & 0.78 & 1.00
         & 1.00 & 1.00 & 1.00 & 1.00 & 1.00
\\
Amazon
         & 1.00 & 0.64 & 1.00 & 0.67 & 1.00
         & 1.00 & \hl{0.10} & 0.75 & 0.75 & 1.00
\\
Coco
         & 1.00 & 1.00 & 0.93 & 0.94 & 0.93
         & 0.80 & \hl{0.00} & 0.80 & 0.50 & 0.50
\\
DBLP
         & 1.00 & 1.00 & 0.64 & 1.00 & 1.00
         & \hl{0.30} & \hl{0.00} & \hl{0.20} & 0.80 & 0.60
\\
CIFAR
         & 1.00 & 1.00 & 1.00 & 0.94 & 1.00
         & \hl{0.00} & \hl{0.00} & 0.50 & \hl{0.10} & \hl{0.40}
\\
Computers
         & 1.00 & 1.00 & 1.00 & 1.00 & 1.00
         & 1.00 & 1.00 & 0.55 & 1.00 & 1.00
\\
\midrule
{\bf Average}
         & {\bf 1.00} & {\bf 0.96} & {\bf 0.91} & {\bf 0.89} & {\bf 0.99}
         & 0.76 & \hl{0.25} & 0.68 & 0.65 & 0.76
\\

\bottomrule
\end{tabular}
\caption{{\em AUC for classifying surrogate versus independent models under model extraction attack:}
This is similar to Table~\ref{tab:dataset-model-results}, except that we use cosine similarity for PreGIP instead of distances.
\ourmethod is comprehensively better, especially for the GIN architecture.
}
\label{tab:dataset-model-results2}
\end{table*}
\section{Experimental Details}
\label{sec:app:details}

Table~\ref{tab:dataset} shows the characteristics of our 14 datasets.
Table~\ref{tab:dataset-model-results2} compares the surrogate detection AUC of \ourmethod and a version of PreGIP where the scores are based on cosine-similarity rather than distances between embeddings of pairs of watermarks graphs.
PreGIP with cosine similarity shows worse results than using distances, so we focus on the latter in the main paper (Table~\ref{tab:dataset-model-results}).

Figure~\ref{fig:finetune_dsacc} shows the normalized accuracy of the surrogate model on a downstream task after being fine-tuned for that task.
We report the trimmed means over 5 datasets.
All models stabilize within 20-40 epochs, with the greatest benefits occurring for GraphSage and MixHop.
Note that \ourmethod's surrogate detection accuracy remains robust throughout fine-tuning (Figure~\ref{fig:allTheRest:finetune}).

Figure~\ref{fig:numstatpts} shows \ourmethod's normalized AUC when we vary the number of stationary points used in Algorithm~\ref{alg:general} (parameter named~$\ell$).
We need only $20-40$ points to achieve a high surrogate detection accuracy.

Figure~\ref{fig:cosine} shows the distribution of cosine similarities between the node features of two stationary points of the same size.
The distribution is mostly concentrated around $0$, showing that most stationary points are distinct from each other.
Hence, Algorithm~\ref{alg:general} picks distinct graphs for its queries.

Figure~\ref{fig:indep_psc} shows the distribution of \ourmethod's score $\hat{\beta}_Z$ for all independent models (Equation~\ref{eq:percent}).
Suppose the stationary points of the victim model are not in any way special points for the independent models.
Then, we would expect the percentile scores to be spread uniformly between $0$ to $100$, with a mean around $\hat{\beta}_Z=50$.
We see that for most datasets and GNN architectures, the $\hat{\beta}_Z$ scores are indeed around $50$.
This provides some evidence for our Assumptions~\ref{A:idiosyncratic} and~\ref{A:q_idio}.
Table~\ref{tab:hyperparams} lists the GNN architectures and configurations used in our experiments.
All seeds are fixed for reproducibility.
We used standard configurations from the literature and did not tune hyperparameters for highest accuracy; our goal was to evaluate whether CopyCop works across a diverse range of trained models.

\textbf{Compute resources:} All experiments were run on a single workstation with an Intel Core i9-10980XE CPU (18 cores / 36 threads), 256 GB of RAM, and 2$\times$ NVIDIA RTX 3090 GPUs (24 GB each). The benchmark spans 14 graph datasets and 5 GNN architectures (GCN, GIN, GraphSAGE, ARMA, MixHop) per dataset, with stolen and independent copies for each base model.
All (dataset, architecture) pairs are independent of one another, so we ran them as independent processes. 
%GPUs were used only for training the base, stolen, and indepenent GNNs but the rest of our pipeline: the per-node stationary-point search, the line-search, and the embedding/gradient/delta evaluations across all candidate models all run on CPU. 
Wall-clock per-stationary-point timings reported in Figure~\ref{fig:time} is per-process (not summed across parallel workers).
 
\begin{table}[H]
\small
\centering

\begin{tabular}{ll}
\hline
\textbf{Architecture} & \textbf{Configuration} \\
\hline
MixHop    & (MixHopConv + ReLU) $\times 2$ + Linear \\
ARMA      & (ARMAConv + ReLU + Dropout) $\times 3$ + Linear \\
GraphSAGE & (SAGEConv + ReLU + Dropout) $\times 2$ + SAGEConv \\
GCN       & (GCNConv + ReLU) $\times 3$ + GCNConv \\
GIN       & (GINConv + ReLU) $\times 2$ + GINConv \\
\hline
\end{tabular}
\vspace{1em}
\caption{{\em GNN model configurations:}
All GNNs output node embeddings, whose accuracy is tested on downstream tasks.
For node-level tasks, a final linear layer produces the output.
For graph-level tasks, we append GlobalMeanPooling + Dropout + Linear.}
\label{tab:hyperparams}
\end{table}

\begin{figure}[h]
    \centering
    \includegraphics[width=0.5\columnwidth]{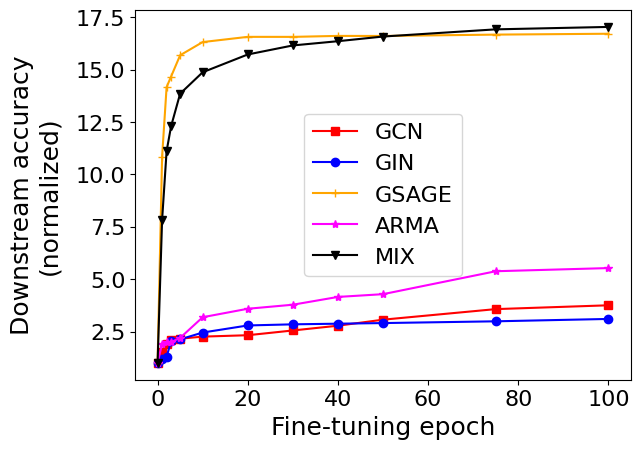}
    \caption{{\em Accuracy of downstream task under fine-tuning:}
    The accuracy is normalized by the accuracy before fine-tuning.
    We report the trimmed mean of the normalized accuracy across 5 datasets.
    The accuracy stabilizes after 20-40 epochs.}
    \label{fig:finetune_dsacc}
\end{figure}

\begin{figure}[h]
    \centering
    \includegraphics[width=0.5\columnwidth]{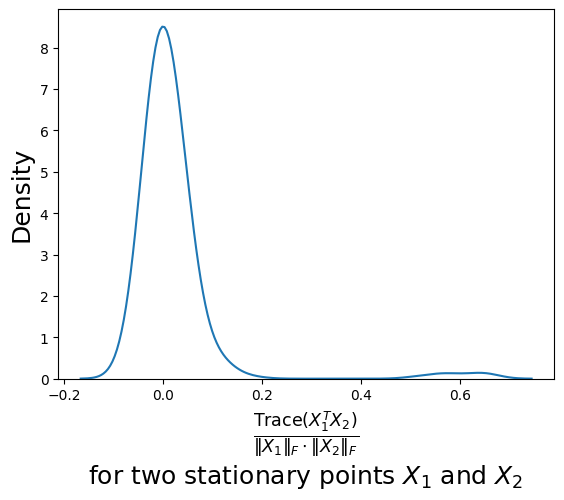}
    \caption{{\em Cosine similarity of stationary points:}
    We picked 25 stationary points for 2-node graphs constructed using Equation~\ref{eq:sample} for the Citeseer dataset.
    For every pair of graphs, we computed the cosine similarity between their node features after flattening the feature matrices into vectors.
    We show the distribution of cosine similarity.
    This shows that most stationary points are nearly orthogonal to each other, showing that our sampling approach in Algorithm~\ref{alg:general} picks distinct stationary points.    
    }
    \label{fig:cosine}
\end{figure}
\begin{figure}
    \centering
    \includegraphics[width=0.5\columnwidth]{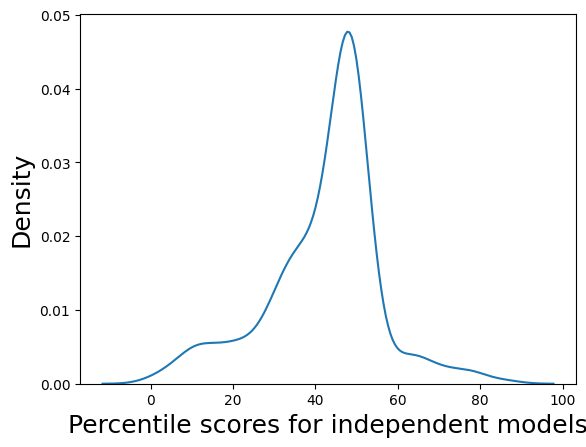}
    \caption{{\em \ourmethod's score distribution for independent models:}
    We show the distribution of $\hat{\beta}_Z$ (Eq.~\ref{eq:percent}) for all independent models across all GNN architectures and datasets.
    Recall that $\hat{\beta}_Z$ is an average of percentile scores.
    The distribution peaks around $50\%$, which is exactly the expected value if the stationary points of the victim model were ``random'' points for the independent models (see Assumptions~\ref{A:idiosyncratic} and~\ref{A:q_idio}).
    }
    \label{fig:indep_psc}
\end{figure}

\begin{table}[H]
\small
\centering
\footnotesize
\begin{tabular}{l@{\hspace{5pt}}c@{\hspace{5pt}}c@{\hspace{5pt}}c@{\hspace{5pt}}c@{\hspace{5pt}}c}
\toprule
& & & \textbf{embed.} & \textbf{integer} & \textbf{avg. nodes}\\
\textbf{Dataset} & \textbf{\#graphs} & \textbf{\#features} & \textbf{dim.} & \textbf{features?} & \textbf{per graph}\\
\midrule
Computers & 5,000 & 767 & 8 & Yes& 31\\
QM9 & 30,000 & 3 & 64 & & 10\\
Amazon & 10,000 & 300 & 8 & & 83\\
Coco & 30,000 & 14 & 8 & & 478\\
PubMed & 8,000 & 500 & 32 & & 6\\
BBBP & 800 & 9 & 8 & Yes & 23\\
HIV & 15,000 & 9 & 8 & Yes & 17\\
OGBMag & 8,000 & 128 & 32 & & 101\\
MNIST & 5,000 & 2 & 8 & & 70\\
Yelp & 5,000 & 300 & 8 & & 60\\
CIFAR & 5,000 & 2 & 8 & & 118\\
DBLP & 8,000 & 1,639 & 8 & & 24\\
Citeseer & 2,000 & 602 & 8 & & 8\\
Fin & 4,000 & 17 & 8 & & 600\\
\bottomrule
\end{tabular}
\caption{{\em Dataset characteristics}}
\label{tab:dataset}
\end{table}

% \newpage
% \input{checklist}

\end{document}